\documentclass{article}

\usepackage{dsfont,amsmath,amssymb,amsthm}
\usepackage{siunitx}
\usepackage{bm}

\usepackage{thmtools}
\usepackage{thm-restate}

\usepackage{hyperref,color}
\definecolor{mydarkblue}{rgb}{0,0.08,0.45}
\hypersetup{ %
pdftitle={Real-Time Optimisation for Online Learning in Auctions},
pdfauthor={Lorenzo Croissant, Marc Abeille, Clément Calauzènes},
pdfsubject={Proceedings of the International Conference on Machine Learning 2020},
pdflang={en-GB},
pdfkeywords={},
pdfborder=0 0 0,
pdfpagemode=UseNone,
colorlinks=true,
linkcolor=mydarkblue,
citecolor=mydarkblue,
filecolor=mydarkblue,
urlcolor=mydarkblue,
pdfview=FitH}

\declaretheorem[name=Theorem]{theorem}
\declaretheorem[name=Proposition]{proposition}
\declaretheorem[name=Lemma]{lemma}
\declaretheorem[name=Corollary]{corollary}
\newtheorem{thm}{Theorem}
\newenvironment{thmbis}[1]
  {\renewcommand{\thethm}{\ref{#1}}%
   \addtocounter{thm}{-1}%
   \begin{thm}}
  {\end{thm}}
  
\newtheorem{assumption}{}

\theoremstyle{definition}
\newtheorem{definition}{Definition}

\theoremstyle{remark}

\newenvironment{proofsketch}{%
  \proof}{\endproof}

\usepackage{graphicx,xcolor}

\usepackage[algoruled,lined,linesnumberedhidden,noend]{algorithm2e}
\DontPrintSemicolon

\usepackage[english]{babel}

\makeatletter
\newcommand{\citethm}[2][]{%
\begingroup
  \let\NAT@mbox=\mbox
  \let\@cite\NAT@citenum
  \let\NAT@space\NAT@spacechar
  \let\NAT@super@kern\relax
  \renewcommand\NAT@open{[}%
  \renewcommand\NAT@close{]}%
  \cite[#1]{#2}%
  \endgroup
}
\usepackage{xspace}

\newcommand{\coga}{\textsc{Conv-OGA}\xspace}
\newcommand{\dcoga}{\textsc{V-Conv-OGA}\xspace}

\newcommand{\NN}{\mathbb{N}}

\newcommand{\RR}{\mathbb{R}}

\newcommand{\EE}{\mathbb{E}}
\newcommand{\II}{\mathds{1}}

\newcommand{\de}{\mathrm{d}}

\newcommand{\norm}[1]{\left\lVert#1\right\rVert}

\newcommand{\argmax}{\text{argmax}}

\newcommand{\mc}[1]{\ensuremath{\mathcal{{#1}}}}

\newcommand{\lN}{\mathbb{N}}

\newcommand{\lR}{\mathbb{R}}
\newcommand{\lP}{\mathbb{P}}
\newcommand{\lE}{\mathbb{E}}
\newcommand{\lI}{\mathds{1}}
\newcommand{\indicator}[1]{{\lI_{#1}}}

\newcommand{\cC}{\mathcal{C}}

\newcommand{\cK}{\mathcal{K}}
\newcommand{\cL}{\mathcal{L}}

\newcommand{\cO}{\mathcal{O}}

\newcommand{\cX}{\mathcal{X}}

\newcommand{\ie}{\textit{i.e. }}
\newcommand{\eg}{\textit{e.g. }}

\newcommand{\toremove}[1]{}

\usepackage[accepted]{icml2020}

\icmltitlerunning{Real-time Optimisation for Online Learning in Auctions}

\usepackage{nicefrac}

\begin{document}

\twocolumn[
\icmltitle{Real-Time Optimisation for Online Learning in Auctions}



\icmlsetsymbol{equal}{*}

\begin{icmlauthorlist}
\icmlauthor{Lorenzo Croissant}{criteo}
\icmlauthor{Marc Abeille}{criteo}
\icmlauthor{Cl\'ement Calauz\`enes}{criteo}
\end{icmlauthorlist}

\icmlaffiliation{criteo}{Criteo AI Lab, Paris, France}

\icmlcorrespondingauthor{Lorenzo Croissant}{ld.croissant@criteo.com}

\icmlkeywords{Machine Learning, ICML}

\vskip 0.3in
]



\printAffiliationsAndNotice{}  

\begin{abstract}

In display advertising, a small group of sellers and bidders face each other in up to $10^{12}$ auctions a day. In this context, revenue maximisation via \emph{monopoly price} learning is a high-value problem for sellers. By nature, these auctions are \emph{online} and produce a very high frequency stream of data. This results in a computational strain that requires algorithms be \emph{real-time}. Unfortunately, existing methods inherited from the batch setting, suffer $\mc{O}(\sqrt{t})$ time/memory complexity at each update, prohibiting their use. In this paper, we provide the first algorithm for online learning of monopoly prices in online auctions whose update is constant in time and memory.

\end{abstract}

\section*{Introduction}\label{sec:intro}

Over the last two decades, online display advertising has become a key monetisation stream for many businesses. The market for the trading of these ads is controlled by a very small ($<\!10$) number of large intermediaries who buy and sell at auction, which means that a seller-buyer pair might trade together in $10^{10}$ to $10^{12}$ auctions per day. Repeated auctions on this scale raise the stakes of revenue maximisation, while making computational efficiency a key consideration. In his \citeyear{myerson1981optimal} seminal work on revenue maximisation,  \citeauthor{myerson1981optimal} described \emph{the} revenue-maximising auction when buyers' bid distributions are known. In the context of online display ads these distributions are private, but the large volume of data collected by sellers on buyers opens the way to \emph{learning} revenue maximising auctions.

\paragraph{Optimal vs. tractable.} The learning problem associated with the \citeauthor{myerson1981optimal} auction has infinite pseudo-dimension \cite{morgenstern2015pseudo}, making it unlearnable \cite{pollard1984convergence}. 2\textsuperscript{nd}-price auctions with personalised reserve prices (\ie different for each bidder) stand as the commonly accepted compromise between optimality and tractability. They provide a 2-approximation \cite{Roughgarden2016} to the revenue of the \citeauthor{myerson1981optimal} auction while securing finite pseudo-dimension.

\paragraph{Monopoly prices.} 2\textsuperscript{nd}-price auctions with personalised reserves can be either \textit{eager} or \textit{lazy}. In the \emph{eager} format, the item goes to the highest bidder \emph{amongst} those who cleared their reserve prices and goes unsold if none of them did. In the \emph{lazy} format, the item goes to the highest bidder \emph{if} he cleared his reserve price and goes unsold if he did not. While an optimised \emph{eager} version leads to better revenue than an optimised \emph{lazy} version, solving the eager auction's associated Empirical Risk Minimisation (ERM) problem is NP-hard \cite{PaesLeme:2016:FGP:2872427.2883071} and even APX-hard \cite{Roughgarden2016}. In contrast, solving the ERM for the \emph{lazy} version can be done in polynomial time \cite{Roughgarden2016}: it amounts to computing a bidder-specific quantity called the \emph{monopoly price}. Not only is the \emph{monopoly price} the optimal reserve in the \emph{lazy} 2\textsuperscript{nd}-price auction, but it is also a provably good reserve in the \emph{eager} one \cite{Roughgarden2016}, and the optimal reserve in posted-price \cite{PaesLeme:2016:FGP:2872427.2883071}. This makes learning monopoly prices for revenue maximisation an important and popular research direction.

\paragraph{Online learning.} Finding the monopoly price in a repeated 2\textsuperscript{nd}-price auction is a natural sequential decision problem based on the incoming bids. All three aforementioned settings relating to the monopoly price have been studied: posted-price \cite{amin2014repeated,blum2004online}, eager \cite{cesa2014regret,Roughgarden2016,kleinberg2003value}, and lazy which we study \cite{blum2005near,blum2004online,mohri2016learning, rudolph2016objective,bubeck2017online,shen2019learning}. Each setting also corresponds to a different observability structure. The offline problems are well understood, but no online method offers the $\mc{O}(1)$ \emph{efficiency} crucial for real-world settings. We focus, therefore, on the key problem of learning monopoly prices, online and efficiently, in the stationary and non-stationary cases.

\paragraph{Structure of the paper.} We propose a real-time first-order algorithm which makes online learning of monopoly prices computationally feasible, when interacting with stationary and non-stationary buyers. In Sec.~\ref{sec:setting}, we detail the setting and problem we consider. We review, in Sec.~\ref{sec:related.work.challenges}, the existing approaches and stress the challenges of the problem including overcoming computational complexities. Our approach, based on convolution and the $\mc{O}(1)$ Online Gradient Ascent algorithm, is described in Sec.~\ref{sec:smoothing.method}. We study performance for stationary bidders in Sec.~\ref{sec:convergence_stationary} with $1/\sqrt{t}$ convergence rate to the monopoly price, and for non-stationary bidders in Sec.~\ref{sec:tracking} with $\cO(\sqrt{T})$ dynamic regret.

\section{Setting}\label{sec:setting}
A key property of a personalised reserve price in a lazy second price auction is that it can be optimised separately for each bidder \cite{PaesLeme:2016:FGP:2872427.2883071}. For a bidder with bid cdf $F$, the optimal reserve price is the monopoly price $r^*$, \ie the maximiser of the \emph{monopoly revenue} defined as
\begin{equation}
    \Pi^F(r) = r(1-F(r))\,.
    \label{eq:def.monopoly.revenue}
\end{equation}
Thus, without loss of generality, we study each bidder separately in the following repeated game: the seller sets a reserve price $r$ and simultaneously the buyer submits a bid $b \in [0, \bar{b}]$ drawn from his private distribution $F$, whose pdf is $f$. The seller then observes $b$ which determines the instantaneous revenue
\begin{equation}
    p(r,b) = r \indicator{r \leq b}~~~~\text{with}~~~~\EE_{F}[p(r,b)]=\Pi^F(r)\,.
    \label{eq:def.instantaneous.revenue}
\end{equation}
In this work, we consider two settings, depending whether the bid distribution is stationary or not.

\paragraph{Stationary setting.} Stationarity here means $F$ is fixed for the whole game. We thus have a stream of i.i.d. bids from $F$, where the seller aims to maximise her long term revenue
\begin{equation*}
    \lim_{T \rightarrow \infty} \frac{1}{T}\sum_{t=1}^T p(r,b_t) = \Pi^F(r)\,.
\end{equation*}
Or, equivalently, tries to construct a sequence of reserve prices $\{r_t\}_{t\ge 1}$ given the information available so far encoded in the  filtration $\mathcal{F}_t = \sigma(r_1,b_1,\ldots,b_{t-1})$ such that $\Pi^{F}(r_t) \rightarrow \Pi^F(r^*)$ as fast as possible.

\paragraph{Non-stationary setting.} In real-world applications, bid distributions may change over time based on the current context. For example, near Christmas the overall value of advertising might go up since customers spend more readily, and thus bids might increase. The bidder could also refactor his bidding policy for reasons entirely independent of the seller. We relax the stationarity assumption by allowing bids to be drawn according to a sequence of distributions $\{F_t\}_{t\geq 1}$ that varies over time. As a result, the monopoly prices $\{r^*_t\}_{t\geq 1}$ and optimal monopoly revenues $\{\Pi^{F_{t}}(r^*_t)\}_{t\geq1}$ fluctuate and convergence is no longer defined. Instead, we evaluate the performance of an adaptive sequence of reserve prices by its expected dynamic regret 
\begin{equation}
    R(T) = \mathbb{E} \left( \sum_{t=1}^T \Pi^{F_t}(r^*_t) - \Pi^{F_t}(r_t) \right)\,,
    \label{eq:def.regret}
\end{equation}
and our objective is to track the monopoly price as fast as possible to minimise the dynamic regret.

\section{Related work, challenges and contributions}\label{sec:related.work.challenges}
\subsection{Related Work}

\emph{Lazy} 2\textsuperscript{nd}-price auctions have been studied both in batch \cite{mohri2016learning, shen2019learning, rudolph2016objective, PaesLeme:2016:FGP:2872427.2883071} and online \cite{blum2005near, blum2004online, bubeck2017online} settings. All existing approaches aim to optimise, at least up to a precision of $1/\sqrt{t}$, the ERM objective
\begin{equation}
      \Pi^{\hat{F}_t}(r) = r(1-\hat{F}_t(r)) = \frac{1}{t}\sum_{i=1}^t r \indicator{r \leq b_i}.
    \label{eq:def.empirical.monopoly.revenue}
\end{equation}

However, regardless of how well-behaved $\Pi^F$ is, $\Pi^{\hat{F}_t}$ is very poorly behaved for optimisation: it is non-smooth, non-quasi-concave, discontinuous, and is increasing everywhere (see Fig.\ref{fig:monopoly_revenue}, center, dashed). Thus direct optimisation with first order methods is not applicable. Some attempts have been made in the batch setting to optimise surrogate objectives, but ended up with an irreducible bias \cite{rudolph2016objective} or with hyper-parameters whose tuning is as hard as the initial problem \cite{shen2019learning}. 
The classical approach relies on sorting the bids $\{b_i\}_{i=1}^t$ to be able to enumerate $\Pi^{\hat{F}_t}(.)$ linearly over $\{b_i\}_{i=1}^t$ \cite{mohri2016learning, PaesLeme:2016:FGP:2872427.2883071}. A popular improvement in terms of complexity, especially used in online approaches \cite{blum2004online, blum2005near} consists in applying the same principle on a regular grid of resolution $1/\sqrt{t}$, which in the end provides an update with complexity $\cO(\sqrt{t})$ and a memory requirement of $\cO(\sqrt{t})$.

This idea of discretising the bid space was widely adopted in partially observable settings -- \eg online \emph{eager} or online posted-price auctions  -- as it reduces the problem to a multi-armed bandit with its well-studied algorithms \cite{kleinberg2003value, cesa2014regret, Roughgarden2016} at the price of still suffering the same update and memory complexities of $\cO(\sqrt{t})$.

Numerous approaches with adversarial bandits also followed this discretisation approach to adapt Exp3/Exp4 \cite{kleinberg2003value,Cohen2016, bubeck2017online} to all the settings. As a bonus, it also allows to handle the case of non-stationary bidders. However, the work of \citet{amin2014repeated} stresses that bidders cannot behave in an arbitrary way, as they optimise their own objective that is not incompatible with the seller's\footnote{An auction is not a zero-sum game: if the item goes unsold, neither player receives payoff.}. Hence, the non-stationarity mostly comes from the item's value changing over time. This suggest adapting a regular stochastic algorithms (ERM, UCB...), \eg using sliding windows \cite{garivier2011upper, lattimore2018bandit}. 

Non-smooth or non-differentiable objectives such as $\Pi^{\hat{F}}$ have been studied in both stochastic and $0$-order optimisation. In both, convolution smoothing has been employed to circumvent these problems. In \citet{duchi2012randomized} stochastic gradient with decreasing convolution smoothing is studied for the convex case. Unfortunately, very few distributions yield a concave $\Pi^F$. In $0$-order optimisation, the only feedback received for an input is the value of the objective at that input. In this setting, \citet{flaxman2005online} perturb their inputs to estimate a convolved gradient. In contrast, we obtain a closed form and do not need to perturb inputs.

\subsection{Challenges}

The directing challenge of our line of work is to devise an online learning algorithm for monopoly prices with minimal cost, to handle very large real-world data streams. With $10^{10}$ daily interactions in one seller-bidder pair, it is acceptable to forfeit some convergence speed in exchange for feasibility of the algorithm. It is not possible to accept  update complexity or memory requirement scaling with $t$. Our objective is thus to find a method that {\bf 1)} converges to $r^*$ in the stationary setting or has a low regret $R(T)$ in the non-stationary one, {\bf 2)} has $\cO(1)$ memory footprint, and {\bf 3)} computes the next reserve $r_{t+1}$ with  $\cO(1)$ computations.

\paragraph{Computational Complexity.} Unfortunately, none of the previously proposed methods fit these requirements. On one hand, all methods based on solving ERM by sorting \cite{cesa2014regret, Roughgarden2016} need to keep all past bids in memory ($\cO(t)$ dependency) and their update steps require at best $\cO(\sqrt{t})$ computations. On the other hand, adversarial methods such as Exp3 or Exp4 \cite{Cohen2016,bubeck2017online} are designed for finite action space and thus need to discretise $[0,\bar{b}]$ into $\sqrt{t}$ intervals (to keep their regret guarantees), also leading to a complexity of $\cO(\sqrt{t})$ from sampling to compute $r_{t+1}$.

\paragraph{Gradient bias.} First order methods (\eg Online Gradient Ascent a.k.a. OGA) are standard tools in online learning and enjoy $\mc{O}(1)$ update and memory. This makes them great candidates for our problem. OGA requires $3$ ingredients to converge: an objective whose gradients always point towards the optimum\footnote{Pseudo-concave or variationally coherent.}, a gradient estimator with bounded variance, and which is unbiased. Unfortunately, discontinuity of $p$ makes $\nabla p$ a \emph{biased} estimator of $\nabla\Pi^F$. A natural approach is to construct a surrogate for $p$ which has unbiased gradients and preserves the other two conditions. 

\paragraph{Surrogate consistency.} Optimising a surrogate objective inherently creates a \emph{bias}, which has to be reduced over-time. To do so without breaking the convergence of OGA, we must conduct a careful finite time analysis of the algorithm, which is an \emph{analytical} challenge. We must re-analyse classical results (\eg \citet{bach2011non,duchi2012randomized}) for varying objectives: the challenge is to design a bias reduction procedure, and then integrate it into these proofs to show we preserve consistency.

\paragraph{Non-stationarity.} Resolving the above challenges is sufficient to achieve efficient convergence in the stationary setting. However, it is not sufficient in order to track non-stationary bid distributions. Taking a constant surrogate and learning rate, it is possible to adapt the stationary solution to the non-stationary case and keep its computational efficiency. The challenge is to devise this adaptation, and then to derive (sub-linear) regret for it.

\subsection{Contributions}
We propose a smoothing method for creating surrogates in pseudo-concave problems with biased gradients. We use it to create a first-order real-time optimisation algorithm which reduces the surrogate's bias during optimisation. We prove convergence and give rates in the stationary setting and dynamic regret bounds for tracking. In more detail:

\paragraph{Smooth surrogates for first order methods.} We first translate standard auction theory assumptions (\eg increasing virtual value) into properties of generalised concavity of the monopoly revenue (Prop.~\ref{prop:pseudo.log.concave.monopoly.revenue}). Next, we introduce our smoothing method and show (in Prop.~\ref{prop:preservation.concavity.smoothing}) that it preserves the properties from Prop.~\ref{prop:pseudo.log.concave.monopoly.revenue} while offering arbitrary smoothness, which fixes the biased gradient problem. Finally, we provide controls, via the choice of the kernel, on the bias and variance of the gradient estimates of our surrogate, which is now ready for OGA (Prop.~\ref{prop:control.bias.variance.smoothing}). 

\paragraph{Consistent algorithm for stationary bidders.} 
We construct an algorithm (\dcoga) which performs gradient ascent while simultaneously decreasing the strength of the smoothing over time, reducing the bias to zero. As a result our algorithm almost surely converges to the monopoly price (Thm.~\ref{thm:as_convergence}) while enjoying computational efficiency. Further, under a minimum curvature assumption, we provide the rate of convergence and optimal tuning parameters (Thm.~\ref{thm:stationary_convergence_speed} and Cor.~\ref{cor:stationary_convergence_speed}). At the cost of a slight degradation in convergence speed (from $\cO(1/t)$ to $\cO(1/\sqrt{t})$), our algorithm has update and memory complexity of $\cO(1)$ which is vital for real-world applications. Results are summarised in Table~\ref{tab:summary}.

\begin{table}[ht]
    \centering
    \begin{tabular}{llll}
         & Update & Memory & Convergence\\
         \hline
        ERM & $\cO(t)$ & $\cO(t)$ & $\cO(1/t)$\\
        Discrete ERM & $\cO(\sqrt{t})$ & $\cO(\sqrt{t})$ & $\cO(1/t)$\\
        \dcoga & $\cO(1)$ & $\cO(1)$ & $\cO(1/\sqrt{t})$\\
    \end{tabular}
    \caption{Comparison of our method (\dcoga) against doing ERM at each step and ERM discretised on a grid of resolution $1/\sqrt{t}$, in terms of complexity and convergence -- \ie  $\lVert\Pi^F(r_t)-\Pi^F(r^*)\rVert$.}
    \label{tab:summary}
\end{table}
    
\paragraph{Tracking for non-stationary bidders.}
Contrary to the stationary setting, when tracking we do not decrease the strength of the smoothing over time. When the bias created is smaller than the noise, our algorithm can still achieve sub-linear dynamic regret when tracking changing bid distributions. For reasonably varying distribution, we show a regret bound of $\mc{O}(\sqrt{T})$ (see Thm.~\ref{thm:tracking} and Cor.~\ref{cor:tracking.regret}).

\section{Smooth Surrogate for First Order Methods}\label{sec:smoothing.method}

\begin{figure*}[ht]
\centering
\includegraphics[width=.32\textwidth]{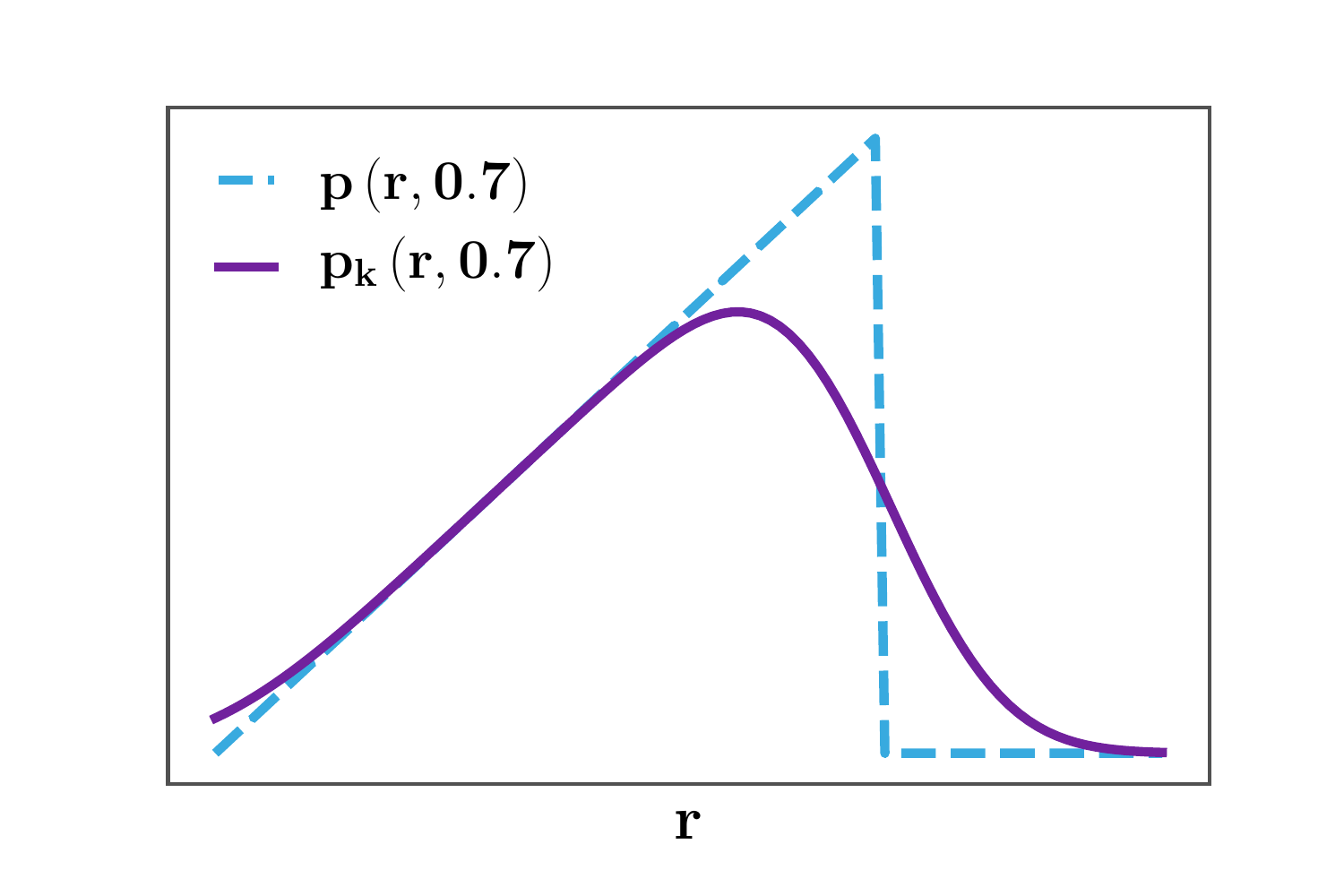}
\includegraphics[width=.32\textwidth]{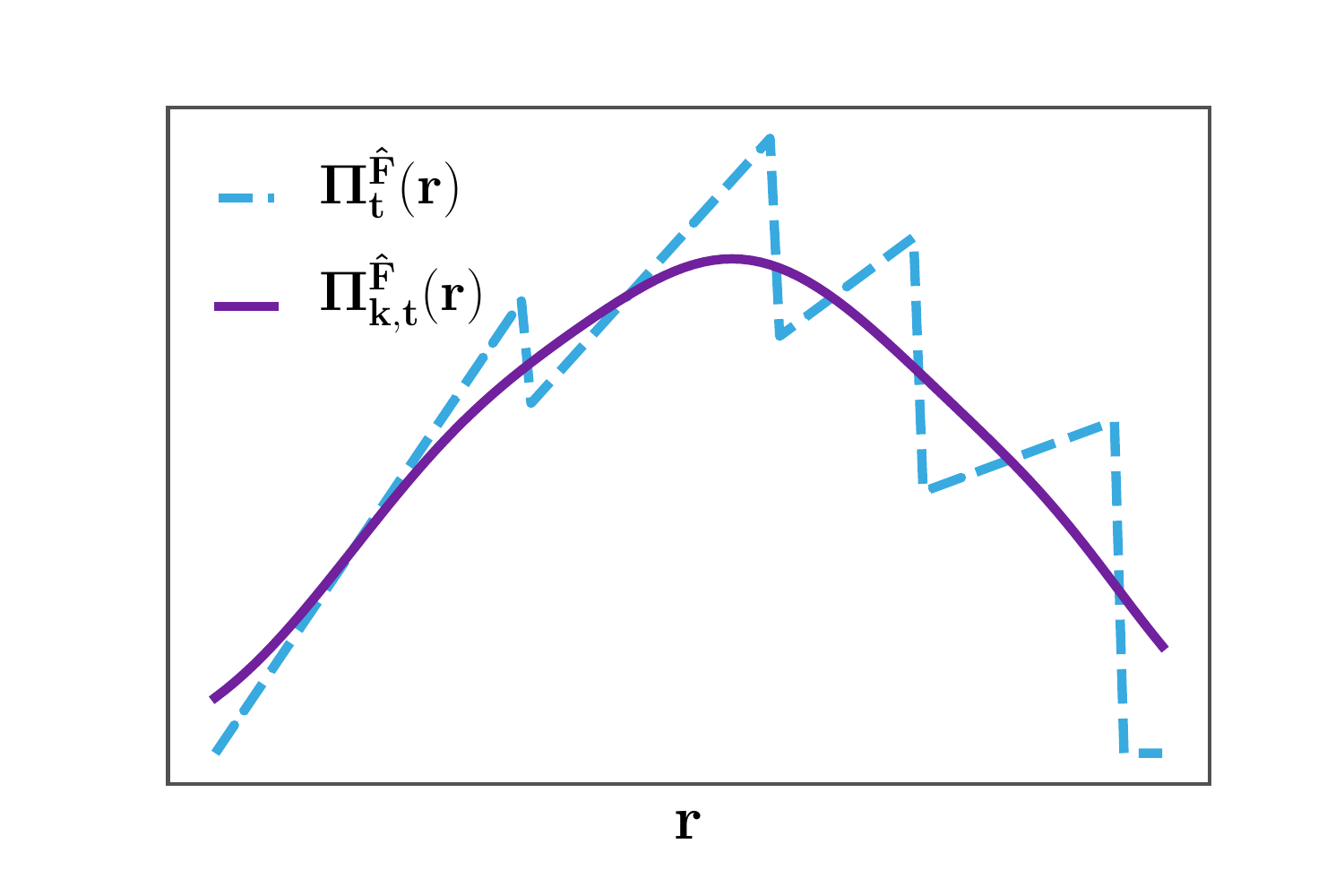}
\includegraphics[width=.32\textwidth]{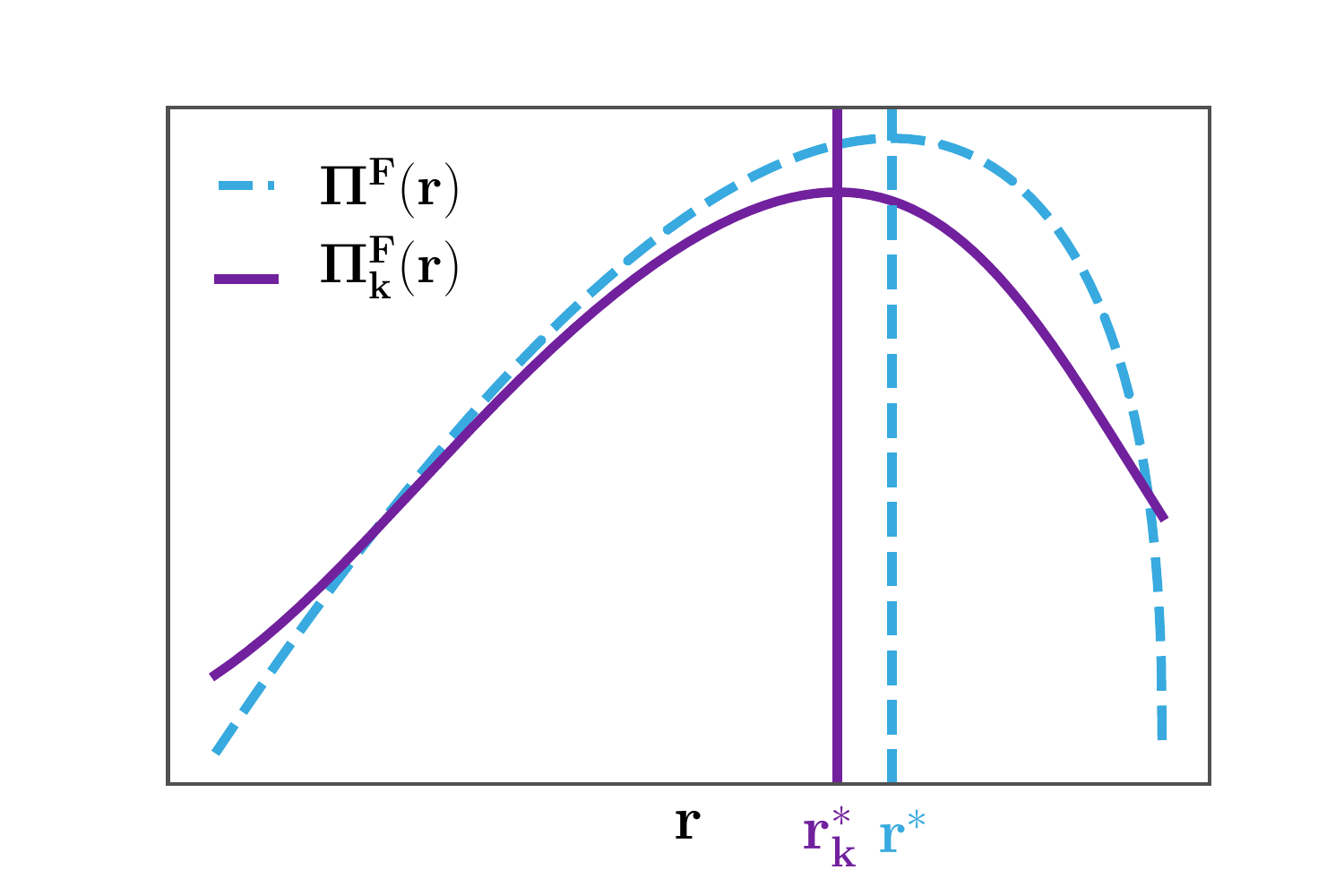}
\caption[]{The effect of smoothing the monopoly revenue of a bidder with $F$ a {Kumaraswamy\footnotemark}(1,0.4) distribution with a Gaussian kernel. Left: smoothing of $p(r,b)$ for $b=0.7$. Center: smoothing of the empirical revenue $\Pi^{\hat{F}}_t(r)$ for some random $b_t$. Right: smoothing of the expected revenue ($\Pi_k^F$ vs $\Pi^F$). Note the difference in the optima (\ie the surrogate bias).}
\label{fig:monopoly_revenue}
\end{figure*}

Our objective, to reiterate, is to design an online optimisation procedure to learn or track the optimal reserve price whose updates require $\mc{O}(1)$ computational and memory cost. To this end, we focus on first order method and consider vanilla Online Gradient Ascent. Unfortunately, the specific problem of learning a monopoly price doesn't provide a way to compute unbiased gradient estimates for $\Pi^F$ from bid samples. We therefore want to design a surrogate that makes $p$ sufficiently smooth so that differentiation and integration commute. This is a well known property of convolutional smoothing, suggesting its use. In addition, we must ensure our surrogate preserves the optimisation properties that $\Pi^F$ already has. These must thus be studied first, before smoothing to obtain a surrogate.

\subsection{Properties of the monopoly revenue}%
\label{subsec:concavity.monopoly.revenue}

The standard assumptions of auction theory are made to guarantee that the monopoly price exists -- and that the optimisation problem is well-posed. This is generally stated as ``the monopoly revenue is \emph{quasi-concave}''. We refine this characterisation by translating the assumptions we make into specific concavity properties of the monopoly revenue in Prop \ref{prop:pseudo.log.concave.monopoly.revenue}. 

\begin{assumption}\label{asmp:finite_variance}
$F \in \cC^{2}\big([0,\bar{b}]\big)$ and $f = F^\prime > 0$ on $(0,\bar{b})$.%
\end{assumption}
\begin{assumption}\label{asmp:regular}
$F$ is strictly regular on its domain of definition \ie the virtual value $\psi(b) = b - \frac{1-F(b)}{f(b)}$ is increasing.
\end{assumption}
\begin{assumption}\label{asmp:mhr}
$F$ has strongly increasing hazard rate on its domain \ie the hazard rate $\lambda(b) = \frac{f(b)}{1-F(b)}$ satisfies:
\begin{equation*}
    \forall 0\leq b_1 \leq b_2 \leq \bar{b},\hspace{1mm} \lambda(b_2) - \lambda(b_1) \geq \mu (b_2 - b_1).
\end{equation*}
\end{assumption}
\ref{asmp:finite_variance} is made for ease of exposition.  \ref{asmp:regular} is a standard auction theory assumption (see~\citet{krishna2009auction} for a review), and implies a pseudo-concave revenue, as shown by Prop \ref{prop:pseudo.log.concave.monopoly.revenue}. \ref{asmp:regular} is satisfied by common distributions, exhaustively listed in \cite{ewerhart2013regular} and by real-world data -- see \eg \cite{Ostrovsky2011}. \ref{asmp:mhr} strengthens \ref{asmp:regular} by requiring a minimum curvature around the maximum.

\begin{restatable}{proposition}{concavityrevenue}\label{prop:pseudo.log.concave.monopoly.revenue}
Let $F$ satisfy~\ref{asmp:finite_variance} and $\Pi^F$ be its associated monopoly revenue. Then, $\Pi^F \in \cC^{2}\big([0,\bar{b}]\big)$, $\Pi^F > 0$ on $(0,\bar{b})$ and:
\begin{itemize}
\item under~\ref{asmp:regular}, $\Pi^F$ is strictly pseudo-concave,
\item under~\ref{asmp:mhr}, $\Pi^F$ is $\mu$-strongly log-concave, i.e. $\log \Pi^F$ is $\mu$-strongly concave on $(0,\bar{b})$.
\end{itemize}
\end{restatable}

\subsection{A Method Based on Smoothing}

Prop.~\ref{prop:pseudo.log.concave.monopoly.revenue} ensures that the first condition for the convergence of OGA is met under standard assumptions~\ref{asmp:regular} or~\ref{asmp:mhr}. The main difficulty in the way of using OGA for revenue optimisation -- it must be stressed -- lies in the undesirable shape of the instantaneous revenue $p$. Indeed, $p$ is non-smooth (discontinuous even) and cannot be used to construct an unbiased estimate of $\nabla \Pi^F(r)$, which is necessary for first order-methods.

\citet{mohri2016learning} suggests replacing $p(\cdot,b)$ by a continuous upper bound. This surrogate can be used for OGA, but it has potentially large areas of zero-gradient, which means it doesn't learn from all samples. We give a general surrogate construction (based on convolutional smoothing) which \textbf{1)} can approximate the original monopoly revenue to arbitrary accuracy, \textbf{2)} preserves the concavity properties of $\Pi^F$, \textbf{3)} offers the desired level of smoothness, and \textbf{4)} exhibits no areas of zero gradient.

Formally, given a kernel $k$ (a metaparameter), we use convolution smoothing to create surrogates for $p$ and $\Pi^F$:
\begin{equation}
    p_k(r,b) = (p(\cdot,b)\star k)(r)\,\, , \,\, \Pi_k^F(r) = (\Pi^F \star k)(r).\label{eq:revenue.surrogate}
\end{equation}
This smoothing guarantees that $\nabla p_k(\cdot,b)$ is an unbiased estimate of $ \nabla \Pi_k^F$. On Fig.~\ref{fig:monopoly_revenue}, we illustrate the effect of this smoothing on $p$, $\Pi^{\hat{F}_t}$, and $\Pi^F$.
We introduce a set $\cK$ of admissible kernels  which contains all strictly positive, strictly log-concave, $\cC^1(\lR)$, $\cL^1(\lR)$, normalised (\ie $\int_\lR k(x)\de x=1$) functions. $\cK$ contains a large family of kernels, including standard smoothing ones such as Gaussians and mollifiers. Prop. \ref{prop:preservation.concavity.smoothing} shows that convolution with elements of $\mc{K}$ preserves pseudo- and log-concavity.

\begin{restatable}{proposition}{concavitysmoothing}
\label{prop:preservation.concavity.smoothing}
    Let $F$ satisfy~\ref{asmp:finite_variance} and $\Pi^F$ be its associated monopoly revenue. Let $k \in \cK$, then:
    \begin{itemize}
        \item $\Pi_k^F$ and $p_k$ are $\cC^1(\mathbb{R})$,
        \item $\Pi_k^F(r)\!\!=\!\!\mathbb{E}_{F} \big( p_k(r,b) \big)$ and $\nabla \Pi_k^F(r)\!\!=\!\!\mathbb{E}_{F} \big(\nabla p_k(r,b) \big)$,
        \item under~\ref{asmp:regular}, $\Pi_k^F$ is strictly pseudo-concave on $\mathbb{R}$,
        \item under~\ref{asmp:mhr}, $\Pi_k^F$ is strictly log-concave on $\RR$.
    \end{itemize}
\end{restatable}
\begin{proof}
See App.~\ref{app:unbiased.preservation.concavity}.%
\end{proof}

%
%

\begin{algorithm}[ht]
\SetAlgoLined
\SetKwInOut{Input}{input}
\Input{$r_{0}$, $\{\gamma_t\}_{t\in\lN}$, $k \in \cK$, $\cC \subset [0, \bar{b}]$}
 \For{$t = 1$ \KwTo $+\infty$}{
 observe $b_t$\;
 $r_{t} \leftarrow {\rm proj}_{\cC} \left(r_{t-1} + \gamma_t \nabla p_{k}(r_{t-1}, b_t)\right)$\;
 }
 \caption{\coga}\label{alg:conv-oga}
\end{algorithm}
Prop. \ref{prop:preservation.concavity.smoothing} guarantees that the surrogate satisfies the pseudo-concavity and unbiased gradient conditions of  OGA. Applying OGA to the surrogate $\Pi_k^F$ gives Alg. \ref{alg:conv-oga}. Note that as a property of convolution, $\nabla p_k(\cdot,b)= p(\cdot, b)\star \nabla k$, which is a simple (generally closed-form) computation.

Prop. \ref{prop:control.bias.variance.smoothing} will show that the bounded variance condition of OGA also holds. Since OGA's three conditions are satisfied, we can guarantee convergence to the maximum $r_k^*$ of $\Pi_k^F$ (see \eg \cite{bottou1998online}). However, in general $r_k^*$ is not the monopoly price $r^*$ and the surrogate is biased. Prop. \ref{prop:control.bias.variance.smoothing} also gives a control on this bias in terms of the $L1$ distance between the cdf $K$ of $k$ and the cdf $\II_{\RR^+}$ of the Dirac mass $\delta_0$, which is the only kernel to guarantee $r_{\delta_0}^*=r^*$.

\footnotetext{This distribution satisfies our concavity assumptions and can display highly eccentric behaviour for easy visualisation of the impact of the surrogate.}

\begin{restatable}{proposition}{controlbiasvariance}%
\label{prop:control.bias.variance.smoothing}
    Let $F$ satisfy~\ref{asmp:finite_variance}, $k \in \cK$. Let $r^*$ and $r^*_k$ be the monopoly prices associated with $\Pi^F$ and $\Pi_k^F$. Then, the bias $B_k = |\Pi^F(r^*) - \Pi^F(r^*_k)|$ and the instantaneous convolved gradient second moment $V_k = \max_{r\geq 0} \mathbb{E}_{b \sim F} \big(| \nabla p_k(r,b)|^2 \big)$ are upper bounded by 
    \begin{itemize}
        \item $B_k \leq 2 \|\nabla \Pi^F\|_\infty \lVert K - \indicator{\lR^+}\rVert_1$,
        \item $V_k \leq 1 + \bar{b} \big(1+\|\nabla \Pi^F \|_\infty\big) \|k\|_\infty$.
    \end{itemize}
\end{restatable}
If one chooses a family of kernels, these bounds can be expressed in terms of its parameters. For instance, when $k$ is zero-mean Gaussian with variance $\sigma^2$, one easily recovers: 
\begin{equation}
    \lVert K - \indicator{\lR^+}\rVert_1 = \sigma\sqrt{2/\pi}\,\, ,\,\, \|k\|_\infty = (\sqrt{2\pi}\sigma)^{-1}.%
    \label{eq:rate.gaussian.kernel.example}
\end{equation}

\coga converges only to $r_k^*$. To remedy this, we would like to decrease $B_k$ over time by letting $k\to\delta_0$. However, since $\norm{k}_\infty$ diverges as $k\to \delta_0$, we will have to tread carefully in our analysis which occupies the next section. 

\section{Convergence with Stationary Bidder}\label{sec:convergence_stationary}

To decrease $B_k$ over time, we introducing a decaying kernel sequence $\{k_t\}_{t\in\NN}$ into \coga, giving \dcoga (Alg. \ref{alg:adapt-conv-oga}). This section will demonstrate its consistency and convergence by controlling the trade-off between bias $B_k$ and variance $V_k$, as $B_k$ is reduced to zero over time. This trade-off decomposes the total error as:
\begin{equation*}
    \Pi^F\!(r^*) - \Pi^F\!(r_t) = \underbrace{\Pi^F\!(r^*) - \Pi^F\!(r_k^*)}_{\text{(surrogate bias)}} + \underbrace{\Pi^F\!(r_k^*) - \Pi^F\!(r_t)}_{\text{(estimation)}}.
\end{equation*}
This stresses that the kernel should converge to $\delta_0$ \textit{fast enough} to cancel the bias $B_k$, yet \textit{slowly enough} to control $V_k$ and preserve the convergence speed of OGA. 
\begin{algorithm}[ht]
\SetAlgoLined
\SetKwInOut{Input}{input}
\Input{$r_0$, $\{\gamma_t\}_{t\in\lN}$, $\{k_t\}_{t\in\lN} \in \cK^\lN$, $\cC \subset [0,\bar{b}]$}
 \For{$t = 1$ \KwTo $+\infty$}{
 observe $b_t$\;
 $r_t \leftarrow {\rm proj}_{\cC} \left(r_{t-1} + \gamma_t \nabla p_{k_t}(r_{t-1}, b_t)\right)$\;
 }
 \caption{\dcoga}
 \label{alg:adapt-conv-oga}
\end{algorithm}

\subsection{General Convergence Result}

Thm.~\ref{thm:as_convergence} provides sufficient conditions on the schedules of $k_t$ and $\gamma_t$ that guarantees \dcoga  converges \textit{a.s.} to $r^*$. It is derived by adapting stochastic optimisation results (see \eg  \citet{bottou1998online}) to the changing objective $\Pi_{k_t}^F$.
\begin{restatable}{theorem}{asconvergence}
\label{thm:as_convergence}
Let $F$ satisfy~\ref{asmp:finite_variance} and~\ref{asmp:regular} and $\{k_t\}_{t\in\lN} \in \cK^\lN$. Then, by running \dcoga with $\cC = [0,\bar{b}]$, we have
\[r_t \xrightarrow{a.s.} r^*\]
as long as $\sum_{t=1}^{+\infty} \gamma_t = +\infty$, $\sum_{t=1}^{+\infty} \gamma_t \lVert K_t - \indicator{\lR^+}\rVert_1 < +\infty$ and $\sum_{t=1}^{+\infty} \gamma_t^2 \lVert k_t \rVert_\infty < +\infty$.
\end{restatable}

\begin{proofsketch}
The proof relies on decomposing the error $\norm{r_t - r^*}^2$ into three terms related respectively to the pseudo-concavity of $\Pi^F$, the bias $B_{k_t}$, and the instantaneous gradient second moment $V_{k_t}$. Then, following \citet{bottou1998online}, we use a quasi-martingale argument to ensure the convergence of the stochastic error process. The full proof is available in App.~\ref{app:as.convergence}.
\end{proofsketch}

If a constant kernel sequence were to be used in \dcoga, we would recover the usual stochastic approximation conditions on the step size $\gamma_t$, namely that $\sum_{t=1}^\infty \gamma_t  = + \infty$ and $\sum_{t=1}^\infty \gamma_t^2 < +\infty$. This suggests setting $\gamma_t \propto 1/t$. For such a choice of step-size, Thm.~\ref{thm:as_convergence} asserts convergence if $\sum_{t=1}^\infty \gamma_t \lVert K_t - \indicator{\lR^+}\rVert_1 < +\infty$, which is guaranteed by $k_t \rightarrow \delta_0$. This means $\|k\|_\infty \rightarrow \infty$ as $t\to \infty$, but $\sum_{t=1}^{+\infty} \gamma_t^2 \lVert k_t \rVert_\infty < +\infty$ tells us explicitly how slow our decay must be in terms of the family of kernels. For example in the case of a Gaussian kernel, \eqref{eq:rate.gaussian.kernel.example} implies that a suitable choice of kernel variance is $\sigma_t \propto t^{-\alpha}$ for $\alpha \in (0,1)$.

\subsection{Finite-time Convergence Rates}

While Thm.~\ref{thm:as_convergence} provides sufficient conditions on the kernel sequence $\{k_t\}_{t\in\lN}$ for \dcoga to be consistent, it does not characterise the rate of the convergence, and thus cannot be leveraged to \textit{optimise} the step size $\gamma_t$ and the decay rate of the kernel.

To obtain finite time guarantees on the rate of convergence, we must impose stronger conditions on the monopoly revenue $\Pi^F$. Recall that under \ref{asmp:regular}, $\Pi^F$ is strictly pseudo-concave. It is well known that such functions can have large areas of arbitrarily small gradient. Since they can make first order methods arbitrarily slow, no meaningful rate can be obtained for them. Strengthening the assumption to \ref{asmp:mhr}, \ie excluding vanishing gradients by ensuring $\Pi^F$ is $\mu$-strongly log-concave (see Prop.~\ref{prop:pseudo.log.concave.monopoly.revenue}), will give a rate in Thm. \ref{thm:stationary_convergence_speed} under the further technical assumption \ref{asmp:lower.bounded.revenue}.

\begin{assumption}
\label{asmp:lower.bounded.revenue}
The seller is given a compact subset $\cC \subseteq [0, \bar{b}]$ and a constant $c > 0$ such that $r^* \in \cC$ and for all $r \in \cC$, $\Pi^F(r) \geq c$.
\end{assumption}

\ref{asmp:lower.bounded.revenue} ensures that the seller can lower bound revenue on a compact subset of $[0,\bar{b}]$. It should be understood as prior knowledge of the seller based on the format of the auction and the type of item sold. $\cC$ exists for any $c < \Pi^F(r^*)$, so this hypothesis is not restrictive relative to \ref{asmp:mhr}.

\begin{restatable}{theorem}{speedconvergence}
\label{thm:stationary_convergence_speed}
Let $F$ satisfy~\ref{asmp:finite_variance} and~\ref{asmp:mhr}, let $\cC$ and $c$ as in~\ref{asmp:lower.bounded.revenue} be given, and let $\{k_t\}_{t\in\lN} \in \cK^\lN$ such that there are $\nu_1$, $\nu_\infty$, $\alpha_1$, $\alpha_\infty$ with
\begin{equation*}
    \lVert K_t - \indicator{\lR^+}\rVert_1 \leq \nu_1 t^{-\alpha_1} \text{ and} \quad \lVert k_t \rVert_\infty  \leq \nu_\infty t^{\alpha_\infty}.
\end{equation*}
Then, by running \dcoga on $\cC$ for $\gamma_t = \nu t^{-\alpha}$ with $\nu \leq (2c\mu)^{-1}$, we have for all $t\geq 2$,
\begin{equation*}
\begin{aligned}
   &\alpha =1:  \hspace{-3mm}&\mathbb{E}(\|r_t - r^*\|^2) &= \tilde{O}\big(t^{-\alpha_1} + t^{\alpha_\infty - 1} + t^{-2 \mu c \nu }\big)
   \\
   &\alpha \in (0,1): \hspace{-3mm}&\mathbb{E}(\|r_t - r^*\|^2) &= \tilde{O}\big( t^{-\alpha_1} + t^{\alpha_\infty - \alpha} \big) \\
   &&+\tilde{O}\Big( \big(t^{1-\alpha - \alpha_1} &+ t^{1 + \alpha^\infty - 2 \alpha} \big) e^{-\mu c \nu t^{1-\alpha}} \Big)
\end{aligned}
\end{equation*}
where $\tilde{O}$ potentially hides a logarithmic term depending on the values of $\alpha$, $\alpha_1$, and $\alpha_\infty$.
\end{restatable}

\begin{proofsketch}
The extended statement of Thm.~\ref{thm:stationary_convergence_speed} with explicit constants and its proof are detailed in App.~\ref{app:speed.convergence}. The proof builds on~\citet[Thm.2]{bach2011non}, derived for log-concave functions, and is adapted to our varying kernel approach and its changing objective. In contrast with the proof of Thm.~\ref{thm:as_convergence}, where we only leveraged pseudo-concavity, we show here that~\ref{asmp:mhr}, together with~\ref{asmp:lower.bounded.revenue}, guarantees more refined control on the curvature of $\Pi^F$ around $r^*$:
\begin{equation}
    \forall r \in \cC, \quad (r - r^*)\nabla \Pi^F(r) \leq - \mu c \|r^* - r\|^2.%
    \label{eq:local.strong.concavity}
\end{equation}
This way we can better control the stochastic process $\{\norm{r_t - r^*}^2 \}_{t\in\NN}$: like for Thm.~\ref{thm:as_convergence}, we decompose the error into three terms related to concavity (Eq. \ref{eq:local.strong.concavity}), bias $B_{k_t}$, and instantaneous gradient smoothness $V_{k_t}$. The error is then bounded in expectation by manipulating finite series.
\end{proofsketch}

\begin{figure}[ht]
\centering
\includegraphics[width=.45\textwidth]{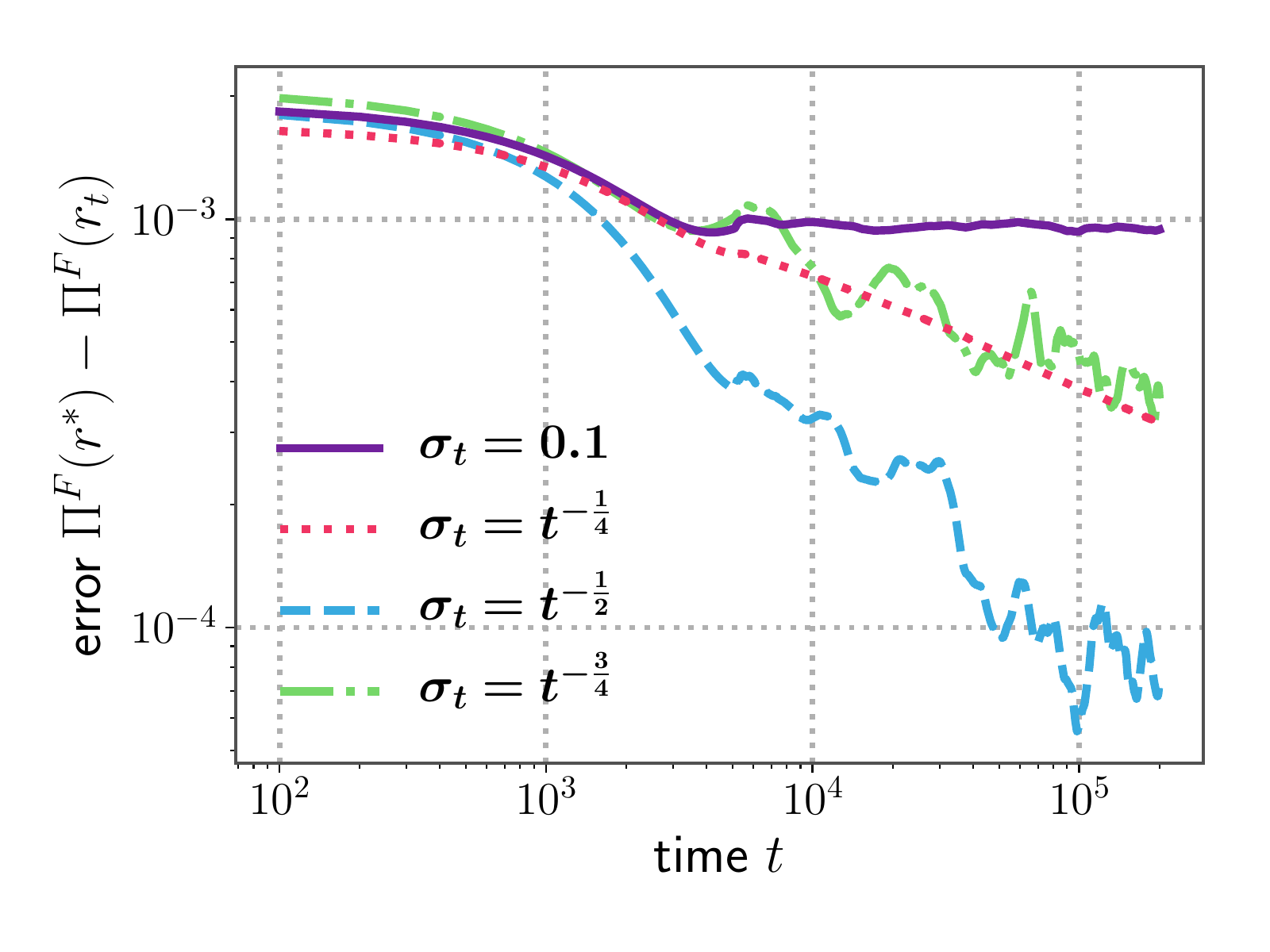}
\includegraphics[width=.45\textwidth]{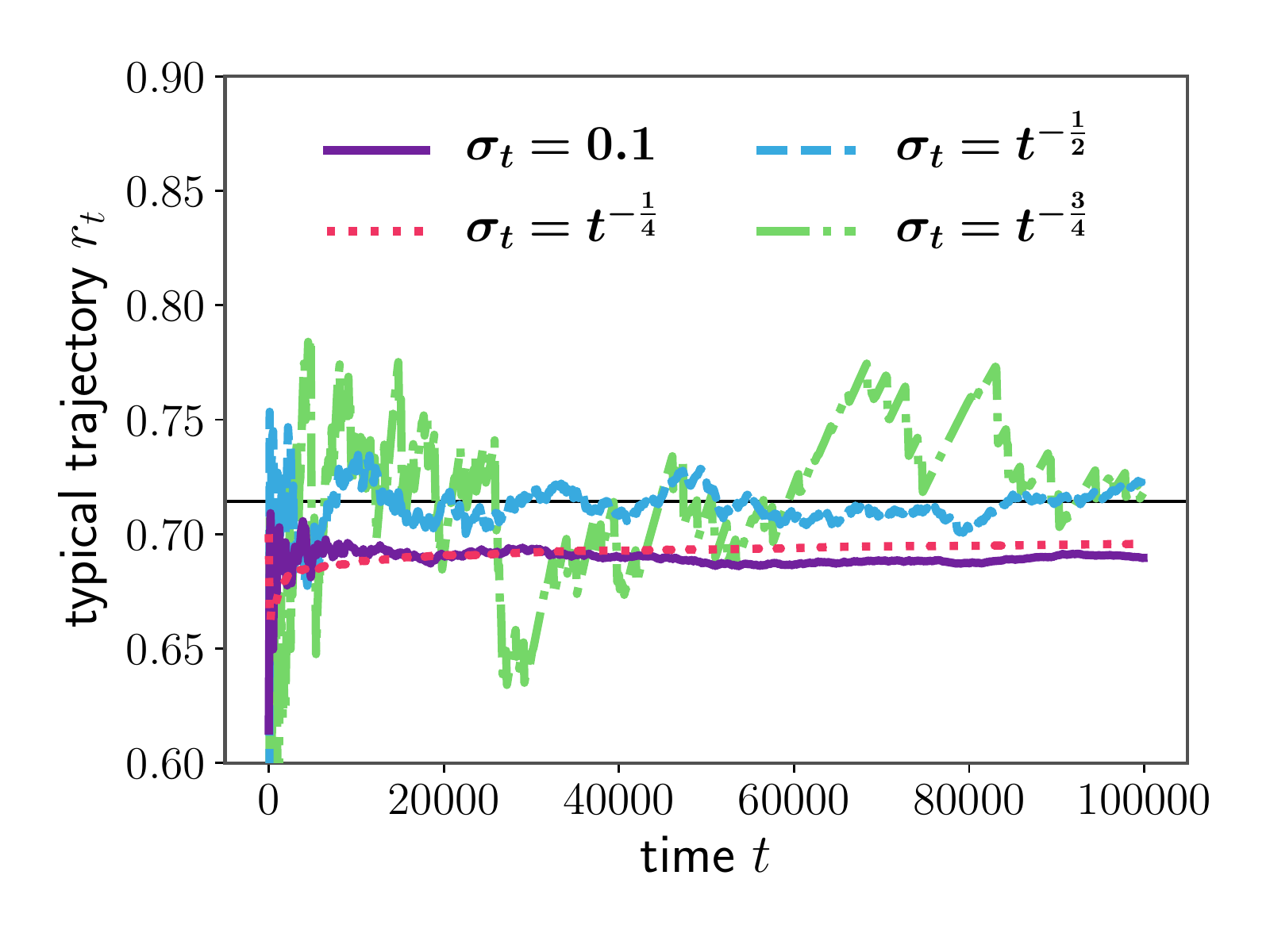}
\caption{Stationary case. Numerical behaviour of \dcoga for different $\sigma_t$ on i.i.d samples from a Kumaraswamy (1, 0.4). Top: averaged convergence speeds of instant regret (log-log scale). Bottom: representative reserve price trajectories.}
\label{fig:convergence_stationary}
\end{figure}

Thm.~\ref{thm:stationary_convergence_speed} show two distinct regimes for both choices of $\alpha$: transient ($t^{- 2 \mu c \nu}$ and $e^{- \mu c \nu t^{1-\alpha}}$ resp.) and stationary ($t^{-\alpha_1} + t^{\alpha_\infty - \alpha}$ and $t^{-\alpha_1} + t^{\alpha_\infty - 1}$ resp.). On Fig. \ref{fig:convergence_stationary} (top), the transient phase is visible up to $2\times10^3$ steps. Since the transient regime's rate depends only on $\nu=\gamma_0$, $c$ known from \ref{asmp:lower.bounded.revenue}, and $\mu$ known from \ref{asmp:mhr}, we can set $\nu$ to make the \emph{stationary} regime the driver of the rate.

To optimise the stationary regime we face a bias-variance trade-off. Like Thm. \ref{thm:as_convergence}, Thm. \ref{thm:stationary_convergence_speed} requires that $k\to\delta_0$ (via $\alpha_1>0$) while imposing a bound on the growth speed of $V_k$ (via $\alpha_\infty<\alpha$). This time, however, we have exact rates which we can use to determine optimal parameters for the trade-off, taking into account the antagonistic effects of $\alpha_1$ and $\alpha_\infty$. From Thm. \ref{thm:stationary_convergence_speed}, we recover that the optimal learning rate is $\gamma_t\propto 1/t$. To tune the kernels it is sensible to fix a parametric family and tune its parameter(s). For zero-mean Gaussian kernels, we have Cor.~\ref{cor:stationary_convergence_speed}.

\begin{restatable}{corollary}{speedconvergenceoptimal}
\label{cor:stationary_convergence_speed}
If we fix $\gamma_t \propto 1/t$, and let $\{k_t\}_{t\in\NN}$ be Gaussian $(0,1/t)$ kernels in Thm. \ref{thm:stationary_convergence_speed} we have for all $t\geq 2$ that:
\[\mathbb{E}\big(\|r_t - r^*\|^2\big) = \tilde{\cO}\big(t^{-1/2}\big)\,.\]
This rate is optimal up to logarithmic factors.
\end{restatable}

Fig. \ref{fig:convergence_stationary} demonstrates this optimality: the $\sigma_t=1/\sqrt{t}$ (blue) curve is the optimal rate on the top pane, and attains the rate of Cor. \ref{cor:stationary_convergence_speed}. The bottom pane illustrates the bias variance trade-off at hand in Thm. \ref{thm:stationary_convergence_speed}. If the kernel decays slower that $1/\sqrt{t}$ (red), the learning rate shrinks much faster and convergence is very slow but very smooth. If $\sigma_t$ decreases too fast (green) the variance becomes overwhelming and noise swallows the performance.

The novel analysis of \dcoga showed its a.s. convergence under \ref{asmp:regular}, and that with a bit of curvature \ref{asmp:mhr} and the technical \ref{asmp:lower.bounded.revenue} we could fully characterise its convergence rates. We could thus derive optimal learning rates and place conditions on optimal kernel decay rates. We made the optimal decay rate explicit for Gaussian kernels. This concludes the primary discussion on \dcoga, and we now move to the non-stationary setting.

\section{Tracking a Non-stationary Bidder}\label{sec:tracking}
In practical applications of online auctions, such as display advertising, bidders might change their bid distribution over time. These changes often result from non-stationarity in the private information of bidders. It is therefore beneficial to be able to effectively adapt one's reserve price over time to \emph{track} changing bid distributions $\{F_t\}_{t\in\NN}$. We use the dynamic regret $R(T)$ to measure the quality of an algorithm's tracking.

The difficulty in the non-stationary setting is to trade-off adaptability (how fast a switch is detected) vs. accuracy (how accurate one is between switches). Convergent algorithms like ERM or \dcoga will have high accuracy in the first phase, but then suffer as they try to adapt to changes later on, when their learning rate is very small. Windowed methods are more adaptable but still carry with them a lag, directly dependent on their window size. First-order methods like \coga (with constant learning rate $\gamma$) are much more adaptable, but their convergence rate ($\mc{O}(1/\sqrt{t})$) hurts their accuracy. Nevertheless, we show that \coga is effective, with $\mc{O}(\sqrt{T})$ regret.

The dynamic regret $R(T)$ cannot be meaningfully controlled for arbitrary sequences $\{F_t\}_{t\in\lN}$. As such it is customary to assume \ref{asmp:tau_switches} that $\{F_t\}_{t\in\lN}$ contains at most $\tau - 1$ switches up to a horizon $T$ (see \eg \cite{garivier2011upper, lattimore2018bandit}). This corresponds to approximating a slowly changing sequence of $F_t$ (\eg Lipschitz) by a piece-wise constant sequence.

\begin{assumption}\label{asmp:tau_switches}
Given some horizon $T$, there exists $\tau \leq T$ such that $\sum_{t=1}^{T-1} \indicator{F_t \neq F_{t+1}} \leq \tau - 1$.
\end{assumption}

Under \ref{asmp:tau_switches}, the game (up to $T$) decomposes into $\tau$ phases. The first step towards controlling the regret is to bound the tracking performance in each phase. We do this in Thm. \ref{thm:tracking}, which shows an incompressible assymptotic error (the bias of our surrogate plus the variance) and a transient phase with exponential decay.

\begin{restatable}{theorem}{trackingconvergence}
\label{thm:tracking}
Let $F$ satisfy~\ref{asmp:finite_variance} and~\ref{asmp:mhr} with parameter $\mu$, let $\cC$ and $c$ be as in \ref{asmp:lower.bounded.revenue} and $k \in \cK$.
Then,  by running \coga on $\cC$ with a constant stepsize $\gamma > 0$, for any $t\geq 1$ we have 
\begin{align*}
    \mathbb{E}(\|r_t - r_k^*\|^2) \leq \big(\bar{b}^2 + C(\gamma, k) (t\,\text{-}\,1)\big) e^{-\frac{\mu c\gamma}{2} t } + \frac{2 C(\gamma, k)}{\mu c}\,
\end{align*}
where $C(\gamma, k) = \cO\left(\gamma \lVert K -\indicator{\lR^+}\rVert_1 + \gamma^2 \lVert k\rVert_\infty\right)$.
\end{restatable}

\begin{figure}[t]
    \centering
    \includegraphics[width=.45\textwidth]{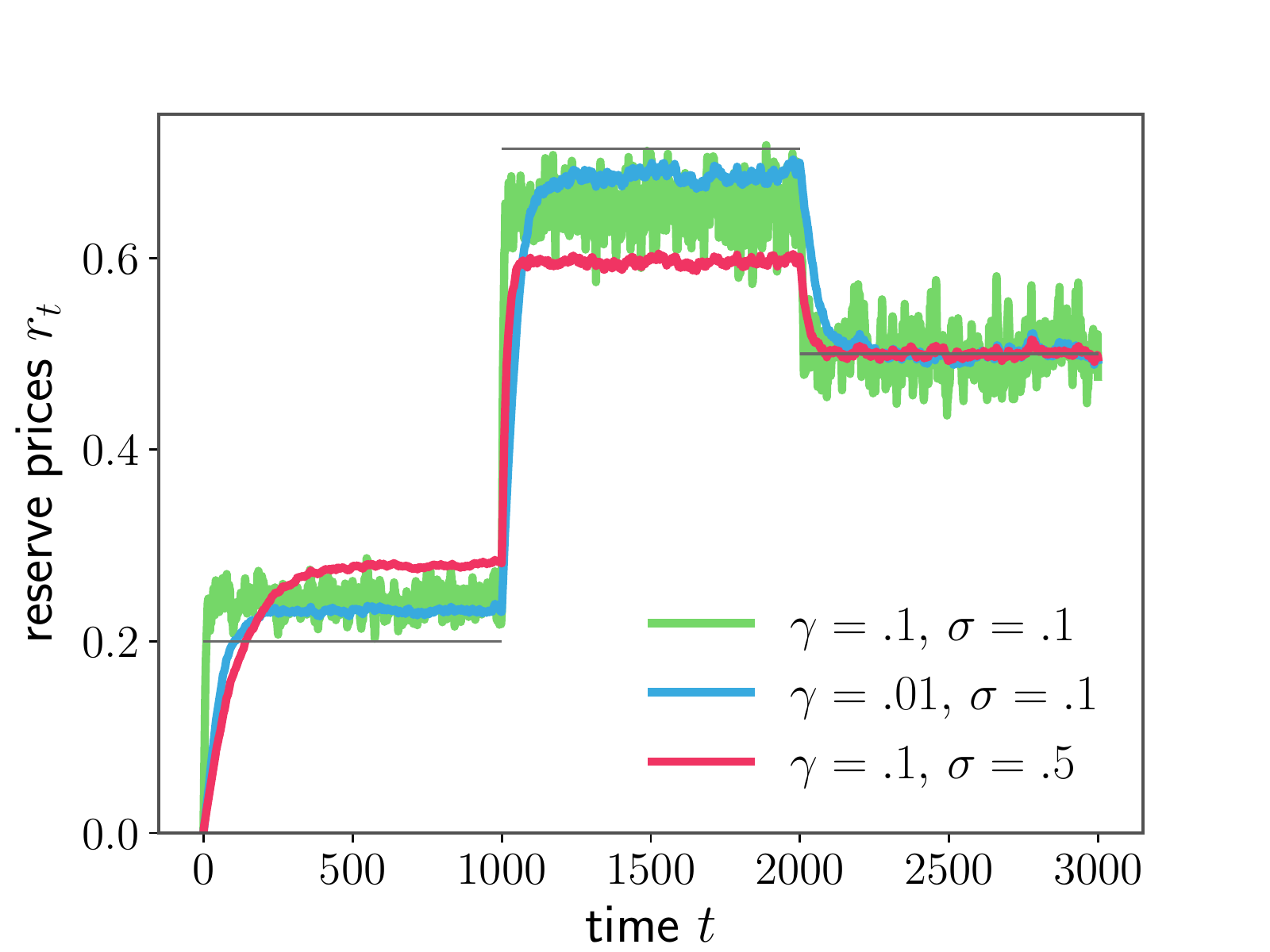}
    \caption[]{Non-stationary case. Tracking by \coga of three Kumaraswamy distributions (with parameters $(1,4)$, $(1, 0.4)$, and $(1,1)$ resp.)  with different Gaussian kernels and learning rates.}\label{fig:tracking.switches}
\end{figure}

Thus, immediately after a switch there will be a transient regime of order $t e^{-\frac{\mu c\gamma}{2}t}$ (high adaptability), but afterwards $r_t$ will oscillate in a band of size $\frac{2C(\gamma,k)}{\mu c}$ around $r^*_k$ (low accuracy).
We can then use Thm. \ref{thm:tracking} to derive a sub-linear regret bound given $T,\tau$ (Cor. \ref{cor:tracking.regret}).

\begin{restatable}{corollary}{trackingregret}
\label{cor:tracking.regret}
Let $\{F_t\}_{t \geq 1}$ satisfy ~\ref{asmp:finite_variance},~\ref{asmp:mhr}, \ref{asmp:lower.bounded.revenue}, \ref{asmp:tau_switches} and $k \in \cK$.
Then, there exists $\Xi(k, \gamma)$ and $\Omega(k,\gamma)$ such that \coga has a non-stationary regret of 
\begin{equation*}
    R(T) \leq \Xi(k, \gamma) T
    +
    \Omega(k,\gamma)\tau\,.
\end{equation*}
Further, if the horizon $T$ is known in advance, running \coga with $\gamma = T^{-\frac{1}{2}}$ and $k$ a  kernel with $\lVert K -\indicator{\lR^+}\rVert_1 \leq T^{-\frac{1}{2}}$ and $\lVert k \rVert_\infty \leq T^{\frac{1}{2}}$, then $R(T) = \cO(\sqrt{T})$.
\end{restatable}
\begin{proof}
See App.~\ref{app:tracking}.%
\end{proof}

Figure \ref{fig:tracking.switches} illustrates the behaviour of \coga in a non-stationary environment. In agreement with Thm.~\ref{thm:tracking} and Cor.~\ref{cor:tracking.regret}, $\gamma$ controls the length of the transient regime due to the $e^{-\frac{\mu c\gamma}{2} t}$ term. Increasing $\gamma$ shortens it but increases the width of the band of the asymptotic regime as $C(\gamma,k)$ increases with $\gamma$ (blue \textit{vs}. green curves). For a fixed $\gamma$, the stationary regime in terms of $k$ exhibits a bias-variance trade-off: $\norm{K-\II_{\RR^+}}_1$ corresponds to the bias and $\norm{k}_\infty$ to the variance (see Prop.~\ref{prop:control.bias.variance.smoothing}). In the case of a Gaussian kernel, increasing $\sigma$ reduces variance but increases bias (green \textit{vs}. red curve). 

\coga with a constant learning rate is an efficient real-time algorithm for tracking monopoly prices of non-stationary bidders. It incurs $\mc{O}(\sqrt{T})$ regret given the horizon and $\tau$, by tuning $\gamma$ and $k$, while maintaining the computational efficiency of online methods.

\vfill

\section{Discussion}\label{sec:discussion}

In this paper we introduced \dcoga, the first \emph{real-time} ($\mc{O}(1)$ update-time and memory) method for monopoly price learning. We first gave some theoretical results bridging auction theory and optimisation. We then showed how to fix the biased gradient problem with smooth surrogates, giving \coga. Next, we let the smoothing decrease over time in \dcoga, for whom we showed convergence of $\mc{O}(1/\sqrt{t})$. Finally, we adapted \coga to perform tracking of non-stationary bid distributions with $\mc{O}(\sqrt{T})$ dynamic regret.

\textbf{Towards optimal rates.} In the context of high-frequency auctions, computational efficiency trumps numerical precision, so we traded $\mc{O}(\sqrt{t})$ complexity and $\mc{O}(1/t)$ speed for $\mc{O}(1)$ complexity and $\mc{O}(1/\sqrt{t})$ speed. Whether or not it is possible to reach the optimal rate with a real-time algorithm remains an open question. We conjecture this to be impossible in general, but we know it is possible in some instances. If $F$ is a symmetric distribution, then \coga with a constant \emph{symmetric} kernel has no bias and $\mc{O}(1/t)$ convergence. Adapting the choice kernel to some \emph{a priori} knowledge on $F$ is a possible direction to match the optimal rate.

\paragraph{Extension to stationary bandit.} The second question concerns the extension to partially observable settings, such as online \emph{eager} auctions, when the seller does not observe bids under the reserve. Obviously, extensions using a reduction to multi-armed bandits (UCB, Exp3, Exp4, \textit{etc}.) via a discretisation of the bid space cannot be \emph{real-time}: the discretisation creates a need for $\cO(\sqrt{t})$ in memory and the same for the update. Yet, it is possible to obtain a strait-forward extension of \dcoga in this setting, by plugging it into an Explore-The-Commit (ETC) \cite{perchet2013} algorithm: \dcoga learns an estimate of the monopoly price during the exploration period, which is then used during the exploitation period. As for other algorithms, by using a doubling trick to handle an unknown horizon, ETC+\dcoga exhibits a sub-linear regret. Unfortunately, like in the \emph{lazy} auction setting, the regret is not optimal and the question of whether a \emph{real-time} algorithm can match this optimal regret is still open.

\paragraph{Extension to non-stationary bandit.} The question of the partially observable setting also applies to non-stationary bidders. In this case, extending \coga with ETC is no longer straight-forward, as the switching times are unknown. Thus, it is not obvious when to re-trigger an exploration phase of ETC to adapt to the change of the bidder's distribution. A potential way to tackle this problem could be to use randomised resets for the algorithm \cite{Allesiardo2017} or change-point detection algorithms to trigger exploration \cite{hartland2006}.

\bibliography{bibliography}

\begin{thebibliography}{29}
\providecommand{\natexlab}[1]{#1}
\providecommand{\url}[1]{\texttt{#1}}
\expandafter\ifx\csname urlstyle\endcsname\relax
  \providecommand{\doi}[1]{doi: #1}\else
  \providecommand{\doi}{doi: \begingroup \urlstyle{rm}\Url}\fi

\bibitem[Allesiardo et~al.(2017)Allesiardo, F{\'e}raud, and
  Maillard]{Allesiardo2017}
Allesiardo, R., F{\'e}raud, R., and Maillard, O.-A.
\newblock The non-stationary stochastic multi-armed bandit problem.
\newblock \emph{International Journal of Data Science and Analytics},
  3\penalty0 (4):\penalty0 267--283, Jun 2017.

\bibitem[Amin et~al.(2014)Amin, Rostamizadeh, and Syed]{amin2014repeated}
Amin, K., Rostamizadeh, A., and Syed, U.
\newblock Repeated contextual auctions with strategic buyers.
\newblock In \emph{Advances in Neural Information Processing Systems}, pp.\
  622--630, 2014.

\bibitem[Bach \& Moulines(2011)Bach and Moulines]{bach2011non}
Bach, F. and Moulines, E.
\newblock Non-asymptotic analysis of stochastic approximation algorithms for
  machine learning.
\newblock In \emph{Advances in Neural Information Processing Systems}, 2011.

\bibitem[Blum \& Hartline(2005)Blum and Hartline]{blum2005near}
Blum, A. and Hartline, J.~D.
\newblock Near-optimal online auctions.
\newblock In \emph{Proceedings of the sixteenth annual ACM-SIAM symposium on
  Discrete algorithms}, pp.\  1156--1163. Society for Industrial and Applied
  Mathematics, 2005.

\bibitem[Blum et~al.(2004)Blum, Kumar, Rudra, and Wu]{blum2004online}
Blum, A., Kumar, V., Rudra, A., and Wu, F.
\newblock Online learning in online auctions.
\newblock \emph{Theoretical Computer Science}, 324\penalty0 (2-3):\penalty0
  137--146, 2004.

\bibitem[Bottou(1998)]{bottou1998online}
Bottou, L.
\newblock Online learning and stochastic approximations.
\newblock \emph{On-line learning in neural networks}, 17\penalty0 (9):\penalty0
  142, 1998.

\bibitem[Bubeck et~al.(2017)Bubeck, Devanur, Huang, and
  Niazadeh]{bubeck2017online}
Bubeck, S., Devanur, N.~R., Huang, Z., and Niazadeh, R.
\newblock Online auctions and multi-scale online learning.
\newblock In \emph{Proceedings of the 2017 ACM Conference on Economics and
  Computation}, pp.\  497--514. ACM, 2017.

\bibitem[Cesa-Bianchi et~al.(2014)Cesa-Bianchi, Gentile, and
  Mansour]{cesa2014regret}
Cesa-Bianchi, N., Gentile, C., and Mansour, Y.
\newblock Regret minimization for reserve prices in second-price auctions.
\newblock \emph{IEEE Transactions on Information Theory}, 61\penalty0
  (1):\penalty0 549--564, 2014.

\bibitem[Cohen et~al.(2016)Cohen, Lobel, and Leme]{Cohen2016}
Cohen, M., Lobel, I., and Leme, R.~P.
\newblock Feature-based dynamic pricing.
\newblock In \emph{Proceedings of the 2016 ACM Conference on Economics and
  Computation}, 2016.

\bibitem[Duchi et~al.(2012)Duchi, Bartlett, and
  Wainwright]{duchi2012randomized}
Duchi, J.~C., Bartlett, P.~L., and Wainwright, M.~J.
\newblock Randomized smoothing for stochastic optimization.
\newblock \emph{SIAM Journal on Optimization}, 22\penalty0 (2):\penalty0
  674--701, 2012.

\bibitem[Ewerhart(2013)]{ewerhart2013regular}
Ewerhart, C.
\newblock Regular type distributions in mechanism design and $\rho$-concavity.
\newblock \emph{Economic Theory}, 53\penalty0 (3):\penalty0 591--603, 2013.

\bibitem[Flaxman et~al.(2005)Flaxman, Kalai, and McMahan]{flaxman2005online}
Flaxman, A.~D., Kalai, A.~T., and McMahan, H.~B.
\newblock Online convex optimization in the bandit setting: Gradient descent
  without a gradient.
\newblock In \emph{Proceedings of the Sixteenth Annual ACM-SIAM Symposium on
  Discrete Algorithms}, SODA ’05, pp.\  385–394, USA, 2005. Society for
  Industrial and Applied Mathematics.
\newblock ISBN 0898715857.

\bibitem[Garivier \& Moulines(2011)Garivier and Moulines]{garivier2011upper}
Garivier, A. and Moulines, E.
\newblock On upper-confidence bound policies for switching bandit problems.
\newblock In \emph{International Conference on Algorithmic Learning Theory},
  pp.\  174--188. Springer, 2011.

\bibitem[Hartland et~al.(2006)Hartland, Gelly, Baskiotis, Teytaud, and
  Sebag]{hartland2006}
Hartland, C., Gelly, S., Baskiotis, N., Teytaud, O., and Sebag, M.
\newblock {Multi-armed Bandit, Dynamic Environments and Meta-Bandits}.
\newblock working paper or preprint, November 2006.
\newblock URL \url{https://hal.archives-ouvertes.fr/hal-00113668}.

\bibitem[Ibragimov(1956)]{ibragimov1956composition}
Ibragimov, I.~A.
\newblock On the composition of unimodal distributions.
\newblock \emph{Theory of Probability \& Its Applications}, 1\penalty0
  (2):\penalty0 255--260, 1956.

\bibitem[Kleinberg \& Leighton(2003)Kleinberg and Leighton]{kleinberg2003value}
Kleinberg, R. and Leighton, T.
\newblock The value of knowing a demand curve: Bounds on regret for online
  posted-price auctions.
\newblock In \emph{44th Annual IEEE Symposium on Foundations of Computer
  Science, 2003. Proceedings.}, pp.\  594--605. IEEE, 2003.

\bibitem[Krishna(2009)]{krishna2009auction}
Krishna, V.
\newblock \emph{Auction theory}.
\newblock Academic press, 2009.

\bibitem[Lattimore \& Szepesv{\'a}ri(2018)Lattimore and
  Szepesv{\'a}ri]{lattimore2018bandit}
Lattimore, T. and Szepesv{\'a}ri, C.
\newblock Bandit algorithms.
\newblock \emph{preprint}, 2018.

\bibitem[Mohri \& Medina(2016)Mohri and Medina]{mohri2016learning}
Mohri, M. and Medina, A.~M.
\newblock Learning algorithms for second-price auctions with reserve.
\newblock \emph{The Journal of Machine Learning Research}, 17\penalty0
  (1):\penalty0 2632--2656, 2016.

\bibitem[Morgenstern \& Roughgarden(2015)Morgenstern and
  Roughgarden]{morgenstern2015pseudo}
Morgenstern, J.~H. and Roughgarden, T.
\newblock On the pseudo-dimension of nearly optimal auctions.
\newblock In \emph{Advances in Neural Information Processing Systems}, pp.\
  136--144, 2015.

\bibitem[Myerson(1981)]{myerson1981optimal}
Myerson, R.~B.
\newblock Optimal auction design.
\newblock \emph{Mathematics of operations research}, 6\penalty0 (1):\penalty0
  58--73, 1981.

\bibitem[Ostrovsky \& Schwarz(2011)Ostrovsky and Schwarz]{Ostrovsky2011}
Ostrovsky, M. and Schwarz, M.
\newblock Reserve prices in internet advertising auctions: A field experiment.
\newblock In \emph{Proceedings of the 12th ACM Conference on Electronic
  Commerce}, EC '11, pp.\  59--60, New York, NY, USA, 2011. ACM.
\newblock ISBN 978-1-4503-0261-6.

\bibitem[Paes~Leme et~al.(2016)Paes~Leme, Pal, and
  Vassilvitskii]{PaesLeme:2016:FGP:2872427.2883071}
Paes~Leme, R., Pal, M., and Vassilvitskii, S.
\newblock A field guide to personalized reserve prices.
\newblock In \emph{Proceedings of the 25th International Conference on World
  Wide Web}, WWW '16, pp.\  1093--1102, Republic and Canton of Geneva,
  Switzerland, 2016. International World Wide Web Conferences Steering
  Committee.

\bibitem[Perchet \& Rigollet(2013)Perchet and Rigollet]{perchet2013}
Perchet, V. and Rigollet, P.
\newblock The multi-armed bandit problem with covariates.
\newblock \emph{The Annals of Statistics}, 41\penalty0 (2):\penalty0 693--721,
  04 2013.

\bibitem[Pollard(1984)]{pollard1984convergence}
Pollard, D.
\newblock \emph{Convergence of stochastic processes}.
\newblock Springer Science \& Business Media, 1984.

\bibitem[Roughgarden \& Wang(2016)Roughgarden and Wang]{Roughgarden2016}
Roughgarden, T. and Wang, J.~R.
\newblock {Minimizing Regret with Multiple Reserves}.
\newblock In \emph{Proceedings of the 2016 ACM Conference on Economics and
  Computation - EC '16}, volume~9, pp.\  601--616, 2016.

\bibitem[Rudolph et~al.(2016)Rudolph, Ellis, and Blei]{rudolph2016objective}
Rudolph, M.~R., Ellis, J.~G., and Blei, D.~M.
\newblock Objective variables for probabilistic revenue maximization in
  second-price auctions with reserve.
\newblock In \emph{Proceedings of the 25th International Conference on World
  Wide Web}, pp.\  1113--1122. International World Wide Web Conferences
  Steering Committee, 2016.

\bibitem[Saumard \& Wellner(2014)Saumard and Wellner]{saumard2014}
Saumard, A. and Wellner, J.~A.
\newblock Log-concavity and strong log-concavity: A review.
\newblock \emph{Statistics Surveys}, 8:\penalty0 45--114, 2014.

\bibitem[Shen et~al.(2019)Shen, Lahaie, and Paes~Leme]{shen2019learning}
Shen, W., Lahaie, S., and Paes~Leme, R.
\newblock Learning to clear the market.
\newblock In \emph{International Conference on Machine Learning}, pp.\
  5710--5718, 2019.

\end{thebibliography}
\bibliographystyle{icml2020}

\onecolumn

\appendix
\section{General Reminders on Pseudo- and Log-Concavity}\label{app:reminders.concavity}

This section is a stand-alone reminder and does not share notations with the rest of the paper.

\subsection{Pseudo-Concavity}
\begin{definition}
A function $f : \cX \to \lR$, $f \in \mc{C}^1(\cX)$, is pseudo-concave on $\cX$ if $\forall (x,y)\in\cX^2$,
\begin{equation*}
\nabla f(x)^T(x - y)\geq 0 \hspace{1mm} \Rightarrow \hspace{1mm} f(x) \geq f(y).
\end{equation*}
\end{definition}

\begin{definition}
A function $f : \cX \to \lR$, $f \in \mc{C}^1(\cX)$, is strictly pseudo-concave on $\cX$ if it is pseudo-concave and has at most one critical point.
\end{definition}

\subsection{Log-Concavity}
\begin{definition}
A function $f : \cX \to \bar{\lR}$ is log-concave on $\cX$ if 
$$\forall \alpha\in[0,1],\forall (x,y)\in\cX^2, f(\alpha x + (1-\alpha)y) \geq f(x)^\alpha f(y)^{(1-\alpha)}$$
Note that if $f : \cX \to \lR^+_*$ this is equivalent to saying $f = e^{-\varphi}$ where $\varphi$ is a convex function on $\cX$.
\end{definition}
\begin{definition}
A function $f : \cX \to \bar{\lR}$ is strictly log-concave on $\cX$ if 
$$\forall \alpha\in(0,1),\forall (x,y)\in\cX^2 \text{ s.t. } x \neq y, \quad f(\alpha x + (1-\alpha)y) > f(x)^\alpha f(y)^{(1-\alpha)}$$
Note that if $f : \cX \to \lR^+_*$ this is equivalent to saying $f = e^{-\varphi}$ where $\varphi$ is a strictly convex function on $\cX$.
\end{definition}
\begin{definition}
A function $f : \cX \to \bar{\lR}$ is $\mu$-strongly log-concave on $\cX$ if $x\mapsto f(x)e^{- \mu x^2}$ is log-concave.
\par Note that if $f : \cX \to \lR^+_*$ this is equivalent to saying 
$f = e^{-\varphi}$ where $\varphi$ is $\mu$-strongly convex on $\cX$.
\end{definition}

We also recall a useful technical result for any log-concave function $f$, that is a straightforward consequence from the concavity characterization of $\log(f)$. 
\begin{proposition}\label{prop:log-conc-3-pentes}
Let $f$ be a real strictly positive strictly log-concave function. Then, for all $ u > v$, for all $\delta > 0$, 
\begin{equation*}
\frac{f(v + \delta)}{f(v)} > \frac{f(u + \delta)}{f(u)}.  
\end{equation*}
\end{proposition}
\begin{proof}[Proof of Prop.~\ref{prop:log-conc-3-pentes}]
The proof is a straightforward application of properties of strictly concave functions applied to $\log(f)$. Let $F(x,y) = \frac{\log f(x) - \log f(y)}{x-y}$, then $F$ is strictly decreasing in $x$ for every fixed $y$ (and vice-versa). Thus, 
\begin{equation*}
\begin{aligned}
F(v + \delta,v) > F(v + \delta, u) > F(u+\delta,u) &\Rightarrow \quad  \log f(v+\delta) - \log f(v) > \log f(u+\delta) - \log f(u)  \\
&\Rightarrow  \quad \frac{f(v + \delta)}{f(v)}  > \frac{f(u + \delta)}{f(u)}. 
\end{aligned}
\end{equation*}

\end{proof}

\subsection{Stability through Convolution}
\label{app:stability.convolution}

\begin{theorem}[\citet{ibragimov1956composition}]
Let $f : \cX \subset \lR \to \lR^+$ be pseudo-concave on $\cX$, $\cC^1(\cX)$, $\cL^1(\cX)$ and $g : \lR \to \lR^+$ be $\cL^1(\lR)$ and log-concave. Then, $f \star g$ is pseudo-concave on $\lR$.
\label{thm:ibragimov}
\end{theorem}

We extend this theorem to \emph{strict} pseudo-concavity.

\begin{lemma}[Extension of \citet{ibragimov1956composition}]
Let $f : \cX \subset [x_1,x_2] \to \lR^+$ be a strictly pseudo-concave on $\cX$, $\cC^1(\cX)$, $\cL^1(\cX)$ such that $\lim_{x\rightarrow x_{1,2}} f(x) = 0$ and $g : \lR \to \lR^+$ be $\cL^1(\lR)$ and strictly log-concave. Then, $f \star g$ is strictly pseudo-concave on $\lR$.
\label{le:ibragimov.extension}
\end{lemma}

\begin{proof}
The proof is conducted in two steps: \textbf{1}) we show $f\star g$ admits a maximum on the interior of its domain (which is a critical point) and we denote it $x^*$. \textbf{2)} we show that $f\star g$ is strictly increasing on $(-\infty,x^*)$ and strictly decreasing on $(x^*,+\infty)$ which immediately proves the strict pseudo-concavity (including unicity of $x^*$).

\begin{enumerate}
\item Since $f$ and $g$ are $\cC^1$ and $\cL^1$, the convolution $f \star g$ is well defined, $\cC^1(\mathbb{R})$, $\cL^1(\mathbb{R})$ and positive (since $f$ and $g$ are positive). As a result, $f\star g(x) \rightarrow 0$ when $|x| \rightarrow \infty$ and Rolle's theorem guarantees that there exists at least one point $x^\star \in \mathbb{R}$ such that $f\star g(x^*) \geq f \star g (x)$ for all $x\in\mathbb{R}$. Furthermore, Ibragimov's theorem (Thm.~\ref{thm:ibragimov}) ensures that $f\star g$ is pseudo-concave, hence $\nabla \big( f \star g \big)(x^*)= 0$.

\item Using the differentiation property of the convolution, one has that for all $x \in \mathbb{R}$,
\begin{equation}
    \nabla \big( f \star g\big) (x) = \int_{-\infty}^\infty f(t) \nabla g(x-t) \de t = \int_{x_1}^{x_2} \nabla f(t) g(x-t) \de t,
    \label{eq:proof.ibragimov.1}
\end{equation}
where we used the fact that $\lim_{x\rightarrow x_{1,2}} f(x) = 0$. Let $x^*$ be a critical point of $f\star g$. Taking Eq.~\ref{eq:proof.ibragimov.1} at $x=x^*$ leads to:
\begin{equation*}
    0 = \int_{x_1}^{x_2} \nabla f(t) g(x^* - t)\de t.
\end{equation*}
Moreover, let $y^* \in (x_1,x_2)$ be the unique (by pseudo-concavity) critical point of $f$, whose existance is guaranteed by Rolle's theorem ($\lim_{x\rightarrow x_{1,2}} f(x) = 0$). We now split the integral in Eq.~\ref{eq:proof.ibragimov.1} to obtain:
\begin{equation}
    \nabla \big( f \star g\big)(x) = \int_{x_1}^{y^*} \nabla f(t) g(x - t)\de t + \int_{y^*}^{x_2} \nabla f(t) g(x - t)\de t.
    \label{eq:proof.ibragimov.2}
\end{equation}

The core of the proof consists in proving that $\nabla \big( f \star g \big)(x^* + \delta) > 0$ for all $\delta >0$ and $\nabla \big( f \star g \big)(x^* + \delta) < 0$ for all $\delta < 0$. Since the derivation is similar in both cases, we only display here the case $\delta > 0$. From Eq.~\ref{eq:proof.ibragimov.2}, we have:
\begin{equation*}
\begin{aligned}
    \nabla \big( f \star g\big)(x^* + \delta) &= \int_{x_1}^{y^*} \nabla f(t) g(x^* +\delta - t)\de t + \int_{y^*}^{x_2} \nabla f(t) g(x^* + \delta - t)\de t \\
    &= \int_{x_1}^{y^*} \nabla f(t) g(x^* -t) \frac{g(x^* +\delta - t)}{g(x^*-t)} \de t + \int_{y^*}^{x_2} \nabla f(t) g(x^* - t) \frac{g(x^* + \delta - t)}{g(x^* - t)}\de t .
\end{aligned}
\end{equation*}
We now provide upper and lower bounds for $\frac{g(x^* + \delta - t)}{g(x^* - t)}$ respectively on $[x_1,y^*]$ and $[y^*,x_2]$. Let $$t^* = \argmax_{t \in [x_1,y^*]} \frac{g(x^* + \delta - t)}{g(x^* - t)}$$
which exists since by our hypotheses on $g$. Then, 
\begin{equation*}
  \forall t \in [x_1,y^*],\quad  \frac{g(x^* + \delta - t)}{g(x^* - t)} \leq \frac{g(x^* + \delta - t^*)}{g(x^* - t^*)}.
\end{equation*}
Moreover, applying Prop.~\ref{prop:log-conc-3-pentes}, we have for almost all $t \in [y^*,x_2)$, 
\begin{equation*}
    \frac{g(x^* + \delta - t)}{g(x^* - t)} > \frac{g(x^* + \delta - t^*)}{g(x^* - t^*)}.
\end{equation*}
Since $f$ is strictly pseudo-concave, $\nabla f(t) >0$ on $[x_1,y^*)$ and $\nabla f(t) < 0$ on $(y^*,x_2]$, we obtain
\begin{equation*}
\begin{aligned}
    \nabla \big( f \star g\big)(x^* + \delta) &= \int_{x_1}^{y^*} \nabla f(t) g(x^* -t) \frac{g(x^* +\delta - t)}{g(x^*-t)} \de t + \int_{y^*}^{x_2} \nabla f(t) g(x^* - t) \frac{g(x^* + \delta - t)}{g(x^* - t)}\de t  \\
    &< \int_{x_1}^{y^*} \nabla f(t) g(x^* -t) \frac{g(x^* +\delta - t^*)}{g(x^*-t^*)} \de t + \int_{y^*}^{x_2} \nabla f(t) g(x^* - t) \frac{g(x^* + \delta - t^*)}{g(x^* - t^*)}\de t \\
    &< \frac{g(x^* + \delta - t^*)}{g(x^* - t^*)} \nabla \big(f\star g\big)(x^*) = 0,
\end{aligned}
\end{equation*}
which proves the desired result.
\end{enumerate}
\end{proof}

Similar stability properties through convolution are asserted for strictly and strongly log-concave functions. The first result is standard and can be derived from the Pr\'epoka-Leindler inequality, the second can be retrieved from~\citep{saumard2014}.
\begin{proposition}
Let $f : \cX \subset \lR \to \lR^+$ and $g : \lR \to \lR^+$ be log-concave. Then, $f \star g$ is log-concave.
\label{prop:log.concave.convolution.log.concave}
\end{proposition}

\begin{theorem}[\citet{saumard2014}, Thm.~6.6]
Let $f : \cX \subset \lR \to \lR^+$ and $g : \lR \to \lR^+$ be
$\mu$ and $\mu^\prime$ strongly log-concave respectively. Then, $f\star g$ is $ \mu \mu^\prime / \sqrt{ \mu^2 + \mu^{\prime 2}}$ strongly log-concave. Further, the convolution of strictly log-concave functions is strictly log-concave.
\label{thm:saumard}
\end{theorem}

\section{Proofs of Sec.~\ref{sec:smoothing.method}}

\subsection{Pseudo- and Log-Concavity of the monopoly revenue}
\label{app:preservation.concavity}
\concavityrevenue*

\begin{proof}\hspace{1mm}
\begin{itemize}
\item Under \ref{asmp:regular}, $\psi(r) = r - \frac{1-F(r)}{f(r)}$ is stricly increasing. Moreover, for all $r \in [0,\bar{b}]$, 
$$\nabla\Pi^F(r) = 1 - F(r) - r f(r) = - \psi(r) f(r).$$ 
The objective is to show that for all $(r_1,r_2) \in [0,\bar{b}]^2$,  $\nabla\Pi^F(r_1)(r_1 - r_2) \geq 0 \hspace{1mm} \Rightarrow \hspace{1mm} \Pi^F(r_1) \geq \Pi^F(r_2)$ and that $\Pi^F$ has one critical point (by Rolle's theorem, since $\Pi^F(0) = \Pi^F(\bar{b}) = 0$). Without loss of generality, we only address the case where $r_1 \leq r_2$. Since $\psi$ is strictly increasing, $\psi(r_1) \leq \psi(r_2)$, and as a result
\begin{equation*}
\begin{aligned}
& \nabla\Pi^F(r_1) \leq  0 \hspace{1mm} \Leftrightarrow \hspace{1mm} \psi(r_1) \geq 0  \\
\Rightarrow  \hspace{2mm}& \forall r \in [r_1,r_2]  \hspace{1mm}, \psi(r) \geq 0 \hspace{1mm} \Leftrightarrow \hspace{1mm} \nabla\Pi^F(r) \leq 0 \\
\Rightarrow  \hspace{2mm}& \Pi^F(r_2) - \Pi^F(r_1)  = \int_{r_1}^{r_2} \nabla\Pi^F(r) {\rm d}r \leq 0.
\end{aligned}
\end{equation*} 
The case $\nabla\Pi^F(r_1) \geq 0$ is treated in a similar fashion. Finally, since $\psi$ is strictly increasing it can only cross $0$ once, which immediately ensures the uniqueness of the critical point since $f>0$ on $(0,\bar{b})$.
\item Under \ref{asmp:mhr}, the hazard rate $\lambda(r) = \frac{f(r)}{1-F(r)}$ satisfies
$\forall 0 \leq r_1 \leq r_2 \leq \bar{b},\hspace{1mm} \lambda(r_2) - \lambda(r_1) \geq \mu (r_2 - r_1)$.
The objective is to show that $\log \Pi^F(r) = \log(r) + \log(1-F(r))$ is $\mu$-strongly concave. As $\log(r)$ is concave, we can simply show that $\log(1-F(r))$ is $\mu$-strongly concave. Since $F\in\cC^1([0,\bar{b}])$, we can use a characterisation of strong concavity of $G(r) = \log(1-F(r))$ based on its derivative:
\begin{equation*}
    G \text{ is $\mu$-strongly concave }\Leftrightarrow \forall (r_1, r_2) \in [0,\beta]^2, \left(\nabla G(r_2) - \nabla G(r_1)\right)^T(r_2 - r_1) \leq -\mu \lVert r_2 - r_1\rVert^2
\end{equation*}
Without loss of generality, we consider the case where $0 \leq r_1 \leq r_2 \leq \bar{b}$,
\begin{align*}
    \nabla G(r_2) - \nabla G(r_1) 
    & = \frac{-f(r_2)}{1-F(r_2)} - \frac{-f(r_1)}{1-F(r_1)} = \lambda(r_1) - \lambda(r_2) \leq -\mu (r_2 - r_1).
\end{align*}
Hence $G$ is $\mu$-strongly concave.
\end{itemize}
\vskip -3em
\end{proof}

\subsection{Unbiased gradient and preservation of concavity}
\label{app:unbiased.preservation.concavity}
\concavitysmoothing*
\begin{proof}
The proof is a straightforward application of the convolution's properties, the Fubini-Tonelli theorem and of the stability of concavity w.r.t. convolution detailed in App.~\ref{app:stability.convolution}.
\begin{enumerate}
    \item Since $k \in \cC^1(\lR) \cap \cL^1(\mathbb{R})$, and since $\Pi^F$ and $p$ are $\cL^1$, $\Pi^F_k$ and $p_k$ are in  $\cC^1 \cap \cL^1 $.
    \item Since $p$, $\Pi^F$, and $k$ are positive, so are $\Pi^F_k$ and $p_k$. Thus, the Fubini-Tonelli theorem ensures that $\Pi^F_k(r) = \mathbb{E}_{b\sim F} \big(p_k(r,b)\big)$. Further, $\nabla \Pi^F_k(r) = \mathbb{E}_{b \sim F} \big( \nabla p_k(r,b) \big)$.
    \item Under~\ref{asmp:regular}, $\Pi^F$ is strictly pseudo-concave (Prop.~\ref{prop:pseudo.log.concave.monopoly.revenue}) and $k \in \cK$ is strictly log-concave. Further, $\Pi^F(0) = \Pi^F(\bar{b}) = 0$ and we can apply Lem.~\ref{le:ibragimov.extension} to guarantee the strict pseudo-concavity of $\Pi^F_k$.
    \item Under~\ref{asmp:mhr}, $\Pi^F$ is strictly log-concave and $k \in \cK$ is strictly log-concave.  Since the convolution preserves  strict-concavity (see Thm.~\ref{thm:saumard}), $\Pi^F_k$ is strictly log-concave.
    \end{enumerate}

\end{proof}

\subsection{Bias and bounded gradient}
\label{app:bias.bounded.gradient}

\controlbiasvariance*
\begin{proof}\hspace{1mm}
\begin{enumerate}
    \item The bound on $B_k$ relies on Lem.~\ref{lem:uniformly_bounded_diff}, which guarantees that for all $r\geq 0$,
    \begin{equation*}
        |\Pi^F(r) - \Pi^F_k(r)| \leq \lVert\nabla\Pi^F\rVert_\infty \int_{-\infty}^\infty |r| k(r) \de r.
    \end{equation*}
    Decomposing $B_k$ as 
    \begin{equation*}
        B_k \leq \Pi^F(r^*) - \Pi^F_k(r^*) + \Pi^F_k(r^*) - \Pi^F_k(r^*_k) + \Pi^F_k(r^*_k) - \Pi^F(r^*_k) \leq \Pi^F(r^*) - \Pi^F_k(r^*)+\Pi^F_k(r^*_k) - \Pi^F(r^*_k)
    \end{equation*}
    and applying two times Lem.~\ref{lem:uniformly_bounded_diff} proves the desired result.

    \item For all $(r,b) \in \mathbb{R}_+ \times [0,\bar{b}]$, using properties of the convolution, one has:
    \begin{equation*}
    \begin{aligned}
        \nabla p_k(r,b) &= \int_{-\infty}^\infty p(\tau,b) \nabla k(r-\tau)\de\tau = \int_{0}^{b} \tau \nabla k (r - \tau) \de\tau \\
    &= \big[ - \tau k(r-\tau) \big]_0^b + \int_0^b k(r-\tau) \de\tau =  \int_0^b k(r-\tau) \de\tau - b k(r-b).
    \end{aligned}
    \end{equation*}
    Since $k > 0$, and $\|k\|_1 = 1$ it is clear that 
    \begin{equation*}
        - b k(r-b) \leq \nabla p_k(r,b) \leq 1 \quad \Rightarrow \quad | \nabla p_k(r,b) |^2 \leq \max\big(1,b^2 k(r-b)^2 \big) \leq 1 + b^2 k(r-b)^2.
    \end{equation*}
     Taking the expectation w.r.t. $F$ (with pdf $f$), one obtains:
     \begin{equation*}
        \lE_B[(\nabla p_k(r,B))^2] \leq 1 + \int_0^{\bar{b}} b^2 k(r-b)^2 f(b) \de b \leq 1 + \|k\|_\infty \int_{0}^{\bar{b}} b^2 f(b) k(r-b) \de b.
    \end{equation*}
    Finally, under \ref{asmp:finite_variance}, for all $b \in [0,\bar{b}]$, 
    \begin{equation*}
      b f(b) = 1 - F(b) - \nabla \Pi^F(b) \quad \Rightarrow \quad b f(b) \leq 1 + \|\nabla \Pi^F \|_\infty < \infty.
    \end{equation*}
    As a result, we have that
    \begin{equation*}
      \lE_B[(\nabla p_k(r,B))^2] \leq 1 + \big(1+\|\nabla \Pi^F \|_\infty\big) \|k\|_\infty \int_0^{\bar{b}} b k(r-b) \de b \leq 1 + \bar{b} \big(1+\|\nabla \Pi^F \|_\infty\big) \|k\|_\infty.
    \end{equation*}
\end{enumerate}
\end{proof}

The following intermediate results provide uniform bounds on the distance between the monopoly revenue (resp. gradient) and the convoluted monopoly revenue (resp. gradient).

\begin{lemma}\label{lem:uniformly_bounded_diff}
Let $F$ satisfies \ref{asmp:finite_variance} and $k \in \cK$ be a convolution kernel. For any $r\in[0,\bar{b}]$, we have
$$
|\Pi^F(r) - \Pi^F_k(r)|\leq  \lVert\nabla\Pi^F\rVert_\infty  \lVert K - \indicator{\lR^+}\rVert_1 = \lVert\nabla\Pi^F\rVert_\infty  \int_{-\infty}^\infty |r| k(r) \de r
$$
\end{lemma}
\begin{proof}
First, let's notice that since $\nabla\Pi^F$ is continuous on the closed interval $[0,\bar{b}]$, it is bounded -- \ie $\lVert\nabla\Pi^F\rVert_\infty < +\infty$. Thus, integrating by parts leads to
\begin{align*}
    |\Pi^F(r) - \Pi^F_k(r)| & = \left|(\Pi^F\star(\delta_0 - k))(r)\right|\\
    & \leq \bigg|\underbrace{\Big[\Pi^F(t)(\indicator{\lR^+}(r-t) - K(r-t))\Big]_{-\infty}^{+\infty}}_{ = 0} - \nabla\Pi^F \star (\indicator{\lR^+} - K)(r) \bigg|\\
    & \leq \left|\nabla\Pi^F \star (\indicator{\lR^+} - K)(r) \right|\\
    & \leq \lVert\nabla\Pi^F\rVert_\infty \lVert K - \indicator{\lR^+}\rVert_1 ~~~\text{ (Young's convolution inequality)}
\end{align*}
Finally, a last integration by parts leads to
\begin{equation*}
    \begin{aligned}
        \lVert K - \indicator{\lR^+}\rVert_1 &= \int_{-\infty}^0 K(r) \de r + \int_{0}^\infty \big(1 - K(r) \big) \de r \\ 
        &= \big[rK(r)\big]_{-\infty}^0 - \int_{-\infty}^0 r k(r) \de r + \big[r(1-K(r))\big]_{0}^\infty + \int_0^\infty r k(r) \de r \\
        &= \int_{-\infty}^\infty |r| k(r) \de r.
    \end{aligned}
\end{equation*}
\end{proof}

Then, another bias that is important to control, is the one of the gradient.
\begin{lemma}\label{lem:uniformly_bounded_grad_diff}
Assuming $F$ satisfies \ref{asmp:finite_variance} and $k \in \cK$ be a convolution kernel. For any $r\in[0,\bar{b}]$, we have
$$
|\nabla\Pi^F_k(r) - \nabla\Pi^F(r)| \leq \lVert\nabla^2\Pi^F\rVert_\infty \lVert K - \indicator{\lR^+}\rVert_1 = \lVert\nabla^2\Pi^F\rVert_\infty \int_{-\infty}^\infty |r| k(r) \de r
$$
\end{lemma}
\begin{proof}
The proof is essentially the same as the one of Lemma \ref{lem:uniformly_bounded_diff}.
First, notice that since $\nabla^2\Pi^F$ is continuous on the closed interval $[0,\bar{b}]$, it is bounded -- \ie $\lVert\nabla^2\Pi^F\rVert_\infty < +\infty$.
\begin{align*}
    |\nabla\Pi^F_k(r) - \nabla\Pi^F(r)| & = \left|(\nabla\Pi^F\star(k - \delta_0))(r)\right|\\
    & = \bigg|\underbrace{\Big[\nabla\Pi^F(x)(K(r-x) - \indicator{\lR^+}(r-x))\Big]_{-\infty}^{+\infty}}_{ = 0} - \nabla^2\Pi^F \star (K - \indicator{\lR^+})(r) \bigg|\\
    & = \left|\nabla^2\Pi^F \star (K - \indicator{\lR^+})(r) \right|\\
    & \leq \lVert\nabla^2\Pi^F\rVert_\infty \lVert K - \indicator{\lR^+}\rVert_1 ~~~\text{ (Young's convolution inequality)} \\
    &= \lVert\nabla^2\Pi^F\rVert_\infty \int_{-\infty}^\infty |r| k(r) \de r.
\end{align*}
\end{proof}

\section{Proofs of Sec.~\ref{sec:convergence_stationary}}
In this section we consider only a stationary $F$, and therefore, for simplicity, we will denote $\Pi^F$ simply by $\Pi$.
\subsection{Almost Sure Convergence}
\label{app:as.convergence}
\asconvergence*
\begin{proof}
The proof inherits a lot from classical methods, see e.g. \citet{bottou1998online}. The main difference lies in the type of ``concavity'' required. The proof in \citet{bottou1998online} is derived for \textit{variationally coherent} functions i.e. those such that 
\begin{equation*}
    \forall t\in\lN,\forall \epsilon > 0, \hspace{1mm} \sup_{(r - r_{k_t}^*)^2 > \epsilon} (r - r_{k_t}^*)^T\nabla\Pi_{k_t}(r) < 0.
\end{equation*}
However, such assumption is clearly violated here since $\nabla\Pi_{k_t}(r) \rightarrow 0$ as $r \rightarrow \infty$. Nevertheless since $\Pi_{k_t}$ is strictly pseudo-concave and strictly positive, one can obtain a similar result. \\

Following~\citet{bottou1998online}, we introduce the Lyapunov process $h_t = \norm{r_t - r^*}^2$. Using the fact that the projection operator is a contraction, one obtains: 
\begin{equation*}
\begin{aligned}
    h_{t+1} &= ({\rm proj}_{C}\big(r_{t} + \gamma_t \nabla p_{k_t}(r_t,b_t) \big) - r_*\big)^2  \\
    &\leq (r_{t} + \gamma_t \nabla p_{k_t}(r_t,b_t) - r^*)^2 \\
    &\leq  h_t + 2 \gamma_t (r_t - r^*)^T \nabla p_{k_t}(r_t,b_t) + \gamma_t^2 \big(\nabla p_{k_t}(r_t,b_t)\big)^2.
\end{aligned}
\end{equation*}
Hence, $h_t$ satisfies the recursion:
\begin{equation*}
    h_{t+1} - h_t \leq 2 \gamma_t (r_t - r^*) \nabla p_{k_t}(r_t,b_t) + \gamma_t^2 \big( \nabla p_{k_t}(r_t,b_t) \big)^2
\end{equation*}
Taking the conditional expectation w.r.t. $\mathcal{F}_{t} = \sigma(b_0,\dots,b_{t-1},r_0,\dots,r_t,\gamma_0,\dots,\gamma_t)$, one obtains:
\begin{align}
\mathbb{E} \left[ h_{t+1} - h_t \big|\mathcal{F}_t \right]
&\leq 2 \gamma_t (r_t - r^*)^T \mathbb{E} \left[ \nabla p_{k_t}(r_t,b_t) \big|\mathcal{F}_t \right] + \gamma_t^2 \mathbb{E} \left[ \big(\nabla p_{k_t}(r_t,b_t) \big)^2 \big| \mathcal{F}_t \right]\nonumber\\
&\leq 2 \gamma_t (r_t - r^*)^T \nabla \Pi_{k_t}(r_t))  + \gamma_t^2 \mathbb{E} \left[ \big(\nabla p_{k_t}(r_t,b_t) \big)^2 \big| \mathcal{F}_t \right] \label{eq:proof.bottou.1}
\end{align}
as Prop.~\ref{prop:preservation.concavity.smoothing} provides that $\mathbb{E} \left[ \nabla p_{k_t}(r_t,b_t) \big|\mathcal{F}_t \right] = \nabla \Pi_{k_t}(r_t)$. Then, we decompose the gradient term to isolate the gradient bias:

\begin{align}
\mathbb{E} \left[ h_{t+1} - h_t \big|\mathcal{F}_t \right]
&\leq
    \underbrace{2 \gamma_t (r_t - r^*)^T \nabla \Pi(r_t)) }_{\leq 0 \text{ by pseudo-concavity}}
    + 
    \underbrace{2 \gamma_t (r_t - r^*)^T (\nabla \Pi_{k_t}(r_t) - \nabla \Pi(r_t)))}_{\text{ bias} }
    + \underbrace{\gamma_t^2 \mathbb{E} \left[ \left(\nabla p_{k_t}(r_t,b_t) \right)^2 \big| \mathcal{F}_t \right]}_{\text{ Bounded gradient } }\label{eq:proof.bottou.2}%
\end{align}
The first term in Eq.~\ref{eq:proof.bottou.2} is negative by the pseudo-concavity of $\Pi$ (Prop.~\ref{prop:pseudo.log.concave.monopoly.revenue}). The second term is bounded by $2 \gamma_t \| \nabla^2 \Pi\|_\infty \|K - \indicator{\lR^+}\|_1$ by Lem.~\ref{lem:uniformly_bounded_grad_diff}. The third term is bounded by $\gamma_t^2 \big( 1 + \bar{b}( 1 + \|\nabla \Pi\|_\infty) \|k\|_\infty \big)$ by Prop.~\ref{prop:control.bias.variance.smoothing}.\\
Using the same quasi-martingale argument as in \citet{bottou1998online}, we have that Eq.~\ref{eq:proof.bottou.2}, together with $\sum \gamma_t \lVert K_t - \II_{\RR^+}\rVert_1 < \infty$ and $\sum \gamma_t^2 \lVert k_t\rVert_\infty < \infty$, implies that
\begin{equation*}
\label{eq:proof.bottou.3}
    h_t \overset{a.s.}{\rightarrow} h_\infty < \infty, \quad \sum \mathbb{E} \left[ h_{t+1} - h_t \big|\mathcal{F}_t \right] < \infty.
\end{equation*}
Thus, using Eq.~\ref{eq:proof.bottou.1}, we have that 
\begin{equation}
   \label{eq:proof.bottou.4} 
    0 \leq \sum \gamma_t (r^* - r_t)^T \nabla \Pi(r_t) < \infty.
\end{equation}
Suppose now that $(r_t - r^*) \overset{a.s.}{\rightarrow} h_\infty \neq 0$, then since $\sum \gamma_t = + \infty$, it would lead to $\sum \gamma_t (r^* - r_t)^T \nabla \Pi(r_t)= + \infty$ which is in contradiction with Eq.~\ref{eq:proof.bottou.4}. As a result,
\begin{equation*}
    (r_t - r^*)^T \nabla \Pi(r_t) \overset{a.s.}{\rightarrow} 0.
\end{equation*}
Finally, since $(r_t - r^*)^2 \overset{a.s.}{\rightarrow} h_\infty < \infty$, necessarily, $\nabla \Pi(r_t) \overset{a.s.}{\rightarrow} 0$ and $r_t\overset{a.s.}{\rightarrow} r^*$.
\end{proof}

\subsection{Finite-time Convergence Speed}
\label{app:speed.convergence}
We provide here the full statement of Thm.~\ref{thm:stationary_convergence_speed}, along with explicit rates and constants. The rates are expressed in terms of an auxiliary function $\varphi$, which allows us to handle all configurations of the step-size and kernel decay schedules. Depending on the value of $\alpha$, $\alpha_1$, $\alpha_\infty$, it may introduce some logarithmic term $\log(t)$. This explains the $\tilde{\cO}$ notation used in the abridged version of Thm.~\ref{thm:stationary_convergence_speed} ( Sec.~\ref{sec:convergence_stationary}), which can easily be recovered from this extended version. 

\begin{thmbis}{thm:stationary_convergence_speed}
Let $F$ satisfy~\ref{asmp:finite_variance} and~\ref{asmp:mhr}, $\cC$ be given in~\ref{asmp:lower.bounded.revenue} and $\{k_t\}_{t\in\lN} \in \cK^\lN$ such that $\lVert K_t - \indicator{\lR^+}\rVert_1 \leq \nu_1 t^{-\alpha_1}$ and $\lVert k_t \rVert_\infty  \leq \nu_\infty t^{\alpha_\infty}$. Then, by running \dcoga on $\cC$ with $\gamma_t = \nu t^{-\alpha}$ with $\nu \leq (2c\mu)^{-1}$, we have for all $t\geq 2$,
\begin{equation*}
\begin{aligned}
   & \text{ if } \alpha =1 \quad & \mathbb{E}(\|r_t - r^*\|^2) &\leq
   \Big(\bar{b}^2 + 2 C \nu \nu_1 \varphi_{2\mu c\nu - \alpha_1}(t) + 2 C_\infty \nu^2 \nu_\infty \varphi_{2 \mu c \nu + \alpha_\infty - 1}(t)\Big) t^{-2 \mu c \nu}
   \\
   & \text{ if } \alpha \in (0,1) \quad &  \mathbb{E}(\|r_t - r^*\|^2) &\leq 
   \Big( \bar{b}^2 + C_1 \nu \nu_1 \varphi_{1 - \alpha - \alpha_1}(t) + C_\infty \nu^2 \nu_\infty \varphi_{1 + \alpha_\infty - 2\alpha}(t) \Big) \exp\big( - \mu c \nu t^{1-\alpha}\big) \\
   &&&+ C_1 \frac{\nu_1}{\mu c} t^{-\alpha_1} + C_\infty \frac{\nu \nu_\infty}{\mu c}t^{\alpha_\infty - \alpha}\\
\end{aligned}
\end{equation*}
as long as $\alpha$, $\alpha_1$, $\alpha_\infty$ satisfy the condition of Thm.~\ref{thm:as_convergence} i.e., $\alpha \leq 1$, $\alpha+\alpha_1 > 1$ and $2 \alpha - \alpha_\infty > 1$. The function $\varphi$ and the constant $C_1$, $C_\infty$ are given by
\begin{equation*}
\varphi_\beta(t) = \log(t) \indicator{\beta=0} + \frac{t^\beta - 1}{\beta}\indicator{\beta\neq0}; \quad
C_1 = 2 \bar{b} \|\nabla^2\Pi\|_\infty; \quad C_\infty = 1 + \bar{b} \big(1+\|\nabla \Pi \|_\infty\big).
\end{equation*}
\end{thmbis}
\begin{proof}
The proof builds on \citet[Thm.~2]{bach2011non} . The main differences are that \textbf{1)} we don't require the local function $p_{k_t}$ to be concave \textbf{2)} we don't rely on the strong concavity of $\Pi_{k_t}$ but on its strong log-concavity and one of its lower bounds \textbf{3)} our objective function varies over time because of the sequence of convolution kernels $\{k_t\}_{t\geq 1}$.\\

We first stress that~\ref{asmp:mhr} together with the lower bounded revenue $\Pi_{k_t}$ leads to some sort of local strong concavity of $\Pi$. From Prop.~\ref{prop:pseudo.log.concave.monopoly.revenue}, the strongly increasing hazard rate ensures that $\Pi$ is strongly log-concave with parameter $\mu$. Further, $\Pi$ admits a unique maximum $r^*$ (since $\Pi(0) = \Pi(\beta) = 0$) such that $r^* \in \cC$ by assumption. As a result, for any $r\in\cC$,
\begin{equation*}
\begin{aligned}
    & (r - r^*)^T \big(\nabla\big(\log \Pi\big)(r) - \nabla\big(\log \Pi\big)(r^*)\big) \leq - \mu \|r^* - r \|^2 \\
    \Rightarrow \quad & (r - r^*)^T \big(\nabla \Pi(r) / \Pi(r) - \nabla \Pi(r^*) / \Pi(r^*)\big) \leq - \mu \|r^* - r \|^2 \\
    \Rightarrow \quad & (r - r^*)^T \nabla \Pi(r) \leq - \Pi(r) \mu \|r^* - r \|^2 \leq - \mu c \|r^* - r \|^2.
\end{aligned}
\end{equation*}
We denote as $\tilde{\mu} = \mu c$ the quantity which plays the role of the strong-concavity parameter in \citet[Thm.~2]{bach2011non}.\\

As for the proof of Thm.~\ref{thm:as_convergence}, we introduce the Lyapunov process $h_t =\norm{r_t - r^*}_2^2$ and its expectation $\bar{h}_t = \mathbb{E}(h_t)$. From Eq.~\ref{eq:proof.bottou.2}, we have
\begin{equation*}
    \mathbb{E}(h_{t+1} - h_t) | \mathcal{F}_t) \leq 2 \gamma_t (r_t - r^*)^T \nabla \Pi(r_t) + 2 \gamma_t \bar{b} \|\nabla^2\Pi\|_\infty \lVert K_t - \indicator{\lR^+}\rVert_1 + \gamma_t^2 \big(1 + \bar{b} \big(1+\|\nabla \Pi \|_\infty\big) \|k\|_\infty\big).
\end{equation*}
Using the local strong-concavity of $\Pi$ and that $\lVert K_t - \indicator{\lR^+}\rVert_1 \leq \gamma^1_t$, $\lVert k_t \rVert_\infty  \leq \gamma^\infty_t$, we obtain:
\begin{equation*}
    \mathbb{E}(h_{t+1} - h_t) | \mathcal{F}_t) \leq 2 \tilde{\mu} \gamma_t h_t  + C_1 \gamma_t \gamma_t^1 + C_\infty \gamma_t^2 \gamma_t^\infty,
\end{equation*}
where $C_1 = 2 \bar{b} \|\nabla^2\Pi\|_\infty$ and $C_\infty =1 + \bar{b} \big(1+\|\nabla \Pi \|_\infty\big)$.\footnote{Without loss of generality, we assume that $\nu_\infty \geq 1$.} Taking the expectation leads to:
\begin{equation}
    \bar{h}_{t+1} \leq (1 - 2 \tilde{\mu} \gamma_t) \bar{h}_t + C_1 \gamma_t \gamma_t^1 + C_\infty \gamma_t^2 \gamma_t^\infty.
    \label{eq:proof.bach.delta.1}
\end{equation}
In line with~\citet{bach2011non}, we split the proof depending whether $\alpha = 1$ or $\alpha \in (0,1)$.

\begin{enumerate}
    \item The case $\alpha = 1$: using that $1-x \leq \exp(-x)$ for all $x\in\mathbb{R}$ and applying the recursion $t$ times in Eq.~\ref{eq:proof.bach.delta.1}, we have
    \begin{equation*}
    \begin{aligned}
        \bar{h}_{t} &\leq \bar{h}_1 \exp\Big( - 2 \tilde{\mu} \sum_{s=1}^{t-1} \gamma_s \Big) + C_1 \sum_{s=1}^{t-1} \gamma_s \gamma_s^1 \exp\Big( - 2 \tilde{\mu} \sum_{\tau = s+1}^{t-1} \gamma_\tau \Big) + 
        C_\infty \sum_{s=1}^{t-1} \gamma_s^2 \gamma_s^\infty \exp\Big( - 2 \tilde{\mu} \sum_{\tau = s+1}^{t-1} \gamma_\tau \Big) \\
        &\leq \bar{h}_1 \exp\Big( - 2 \tilde{\mu} \nu \sum_{s=1}^{t-1} s^{-1} \Big) + \Big(C_1 \nu \nu_1 \sum_{s=1}^{t-1} s^{-\alpha - \alpha_1}  + 
        C_\infty \nu^2 \nu_\infty \sum_{s=1}^{t-1} s^{-2 \alpha + \alpha_0} \Big) \exp\Big( - 2 \tilde{\mu} \nu \sum_{\tau = s+1}^{t-1} \tau^{-1} \Big).
    \end{aligned}
    \end{equation*}
    Further, for all $t\geq 2$,
    \begin{equation*}
    \begin{aligned}
    \sum_{s=1}^{t-1} s^{-1} &\geq \log(t)\\
    \sum_{\tau = s+1}^{t-1} \tau^{-1} &\geq \log(t/s+1)\\
    \end{aligned}
    \end{equation*}
    we obtain that (under $\tilde{\mu}\nu \leq 1/2$):
    \begin{equation*}
    \begin{aligned}
        \bar{h}_t &\leq \bar{h}_1 t^{-2 \tilde{\mu} \nu} + C_1 \nu \nu_1 t^{-2 \tilde{\mu} \nu}  \sum_{s=1}^{t-1} s^{-1 - \alpha_1}(s+1)^{2\tilde{\mu}\nu} + C_\infty \nu^2 \nu_\infty t^{-2 \tilde{\mu} \nu} \sum_{s=1}^{t-1} s^{-2 + \alpha_\infty} (s+1)^{2 \tilde{\mu}\nu}\\
        &\leq \bar{h}_1 t^{-2 \tilde{\mu} \nu} + 2 C_1 \nu \nu_1 t^{-2 \tilde{\mu} \nu} \sum_{s=1}^{t-1} s^{2\tilde{\mu}\nu - 1 - \alpha_1} + 2 C_\infty \nu^2 \nu_\infty t^{-2 \tilde{\mu} \nu} \sum_{s=1}^{t-1} s^{-2 + \alpha_\infty + 2 \tilde{\mu}\nu},\\
    &\leq \Big(\bar{b}^2 + 2 C_1 \nu \nu_1 \varphi_{2\tilde{\mu}\nu - \alpha_1}(t) + 2 C_\infty \nu^2 \nu_\infty \varphi_{2 \tilde{\mu} \nu + \alpha_\infty - 1}\Big) t^{-2 \tilde{\mu} \nu}.
    \end{aligned}
    \end{equation*}    
    \item The case $\alpha \in (0,1)$: applying the recursion $t$ times in Eq.~\ref{eq:proof.bach.delta.1}, we have
    \begin{equation}
    \bar{h}_{t} \leq \bar{h}_1 \underbrace{\prod_{s=1}^{t-1} (1 - 2 \tilde{\mu} \gamma_s)}_{A^1_t} + C_1 \underbrace{\sum_{s=1}^{t-1} \gamma_s \gamma_s^1 \prod_{\tau = s+1}^{t-1} (1 - 2 \tilde{\mu} \gamma_\tau)}_{A^2_t} + C_\infty \underbrace{\sum_{s=1}^{t-1} \gamma_s^2 \gamma_s^\infty \prod_{\tau = s+1}^{t-1} (1 - 2 \tilde{\mu} \gamma_\tau)}_{A^3_t}
    \label{eq:proof.bach.delta.2}
    \end{equation}
    The derivation slightly differs from the case $\alpha=1$. Following~\citet{bach2011non}, one has:
    \begin{equation*}
    \begin{aligned}
    A^1_t & \leq \exp\Big( - 2 \tilde{\mu} \sum_{s=1}^{t-1} \gamma_s \Big) \\
    A^2_t &\leq \frac{\gamma^1_{\lfloor t/2 \rfloor}}{2 \tilde{\mu}} + \exp\Big( -2 \tilde{\mu} \sum_{\tau = \lfloor t/2 \rfloor}^{t-1} \gamma_\tau \Big)\sum_{s=1}^{t-1} \gamma_s \gamma_s^1\\
    A^3_t &\leq  \frac{\gamma_{\lfloor t/2 \rfloor} \gamma^\infty_{\lfloor t/2 \rfloor}}{2 \tilde{\mu}} + \exp\Big( -2 \tilde{\mu} \sum_{\tau = \lfloor t/2 \rfloor}^{t-1} \gamma_\tau \Big) \sum_{s=1}^{t-1} \gamma_s^2 \gamma_s^\infty.
    \end{aligned}
    \end{equation*}
    Using the expression of $\gamma$, $\gamma_1$ and $\gamma_\infty$, where $\alpha_1\leq 1$, together with $\varphi_{1-\alpha}(x) - \varphi_{1 - \alpha}(x/2) \geq x^{1-\alpha}/2$ and $\varphi_{1-\alpha}(t) \geq t^{1 - \alpha} /2$ we have:
    \begin{equation*}
    \begin{aligned}
    A^1_t & \leq \exp\big( -2\tilde{\mu} \nu \varphi_{1-\alpha}(t) \big) \leq \exp\big( - \tilde{\mu} \nu t^{1 - \alpha} \big),\\
    A^2_t &\leq \frac{\nu_1 }{\tilde{\mu}}t^{-\alpha_1} + \nu \nu_1 \exp\big( -\tilde{\mu} \nu t^{1 - \alpha}\big)\sum_{s=1}^{t-1} s^{-\alpha - \alpha_1} \leq 
    \frac{\nu_1 }{\tilde{\mu}}t^{-\alpha_1} + \nu \nu_1 \exp\big( -\tilde{\mu} \nu t^{1 - \alpha}\big) \varphi_{1 - \alpha - \alpha_1}(t), 
    \\
    A^3_t &\leq  \frac{\nu \nu_\infty}{\tilde{\mu}} t^{\alpha_\infty - \alpha} + \nu^2 \nu_\infty \exp\big( -\tilde{\mu} \nu t^{1 - \alpha}\big) \sum_{s=1}^{t-1} s^{\alpha_\infty - 2 \alpha} \leq 
    \frac{\nu \nu_\infty}{\tilde{\mu}} t^{\alpha_\infty - \alpha} + \nu^2 \nu_\infty \exp\big( -\tilde{\mu} \nu t^{1 - \alpha}\big) \varphi_{1 + \alpha_\infty - 2 \alpha }(t).
    \end{aligned}
    \end{equation*}
    Putting everything together, we obtain the final bound for $\alpha \in (0,1)$:
    
    \begin{equation*}
    \bar{h}_t \leq \Big( \bar{b}^2 + C_1 \nu \nu_1 \varphi_{1 - \alpha - \alpha_1}(t) + C_\infty \nu^2 \nu_\infty \varphi_{1 + \alpha_\infty - 2\alpha}(t) \Big) \exp\big( - \tilde{\mu}\nu t^{1-\alpha}\big) + C_1 \frac{\nu_1}{\tilde{\mu}} t^{-\alpha_1} + C_\infty \frac{\nu \nu_\infty}{\tilde{\mu}}t^{\alpha_\infty - \alpha}. 
    \end{equation*}

\end{enumerate}
\end{proof}

\toremove{
\begin{corollary}
Under the condition of Thm.~\ref{thm:stationary_convergence_speed}. If \dcoga is run with $\{k_t\}_{t \in \lN}$ a sequence of Gaussian kernels with variances $\sigma_t \propto 1/\sqrt{t}$ and $\gamma_t \propto 1/t$, we have, a stationary regret of:
\begin{equation*}
    R_{s}(T) \triangleq \sum_{t=1}^T \Pi^F(r^*) - \Pi^F(r_t) = \tilde{\cO}(T^{3/4})\,.
\end{equation*}
\end{corollary}
\begin{proof}
\begin{align*}
    R_{s}(T) &\triangleq \sum_{t=1}^T \Pi^F(r^*) - \Pi^F(r_t)\\
    &\leq \sum_{t=1}^T \lP\left(\Pi^F(r^*) - \Pi^F(r_t) \leq \delta\right) \delta + \lP\left(\Pi^F(r^*) - \Pi^F(r_t) > \delta\right) \bar{b} & \text{for any $\delta > 0$}\\
    &\leq \sum_{t=1}^T \delta + \frac{\lE\left(\Pi^F(r^*) - \Pi^F(r_t)\right)}{\delta} \bar{b} & \text{(Markov's inequality)}\\
    &\leq \sum_{t=1}^T \delta + \frac{\lVert\nabla^2 \Pi^F\rVert_\infty\lE\left(\lVert r^* - r_t\rVert^2\right)}{2\delta} \bar{b} & \text{(Taylor's Theorem)}\\
    &\leq \sum_{t=1}^T \delta + \frac{\bar{b}\lVert\nabla^2 \Pi^F\rVert_\infty}{2\delta}\cO\left(t^{-1/2}\right)  & \text{(Corollary \ref{cor:stationary_convergence_speed})}\\
    & \leq \delta T + \frac{\tilde{\cO}(T^{1/2})}{\delta}\\
    & \leq \tilde{\cO}(T^{3/4}) & \text{by choosing $\delta = T^{-1/4}$.}\\
\end{align*}
\end{proof}
\begin{corollary}
In a bandit setting (bids under the reserve are not observed), under the condition of Thm.~\ref{thm:stationary_convergence_speed}, by running an Explore-then-Commit strategy where the fit during the exploration phase is done by \dcoga ran with $\{k_t\}_{t \in \lN}$ a sequence of Gaussian kernels with variances $\sigma_t \propto 1/\sqrt{t}$ and $\gamma_t \propto 1/t$, we have, a stationary regret of:
\begin{equation*}
    R_{s}(T) \triangleq \sum_{t=1}^T \Pi^F(r^*) - \Pi^F(r_t) = \tilde{\cO}(T^{4/5})\,.
\end{equation*}
\end{corollary}
\begin{proof}
Up to using a doubling trick, we can assume $T$ is known. Denoting $N = T^{\alpha}$ the size of the explorations step, we have 
\begin{align*}
    R_{s}(T) &\leq N\bar{b} + (T-N) \frac{\tilde{\cO}(N^{3/4})}{N}\\
    &\leq \cO(T^{\alpha}) + T^{1-\alpha}\tilde{\cO}(T^{3\alpha/4}) + \tilde{\cO}(T^{3\alpha /4})\\
    &\leq \tilde{\cO}(T^{4/5})&\text{by choosing $\alpha = 4/5$}
\end{align*}
\end{proof}
}

\section{Proof of Sec.~\ref{sec:tracking}}
\label{app:tracking}
\trackingconvergence*
\begin{proof}
Similarly to the one of Th.\ref{thm:stationary_convergence_speed}, this proof builds on the one of \citet{bach2011non}. 

Since $\Pi^F$ is $\mu-$strongly log-concave, one has for all $r\in \cC$, 
\begin{equation}
    \label{eq:proof.bach.strong.conc}
    (r_t - r^*)^T \nabla \Pi^F(r) \leq - \Pi^F(r) \mu \| r - r^*\|^2 \leq - \tilde{\mu} \| r - r^*\|^2 
\end{equation}
where $\tilde{\mu} = \mu c$. As a result, although we do not assume the function to be strongly concave, it still enjoys a similar property in $r^*$ on the bounded subset $\cC$. Then, let $h_t = \norm{r_t - r^*}^2$ be the Lyapunov process (similarly to the proof of Thm.~\ref{thm:as_convergence}). Since the projection operator over $\cC$ is $1-$Lipschitz, from Eq. \ref{eq:proof.bottou.2}, one has: 
\begin{equation}
    \label{eq:proof.bach.1}
    \begin{aligned}
    \mathbb{E} \left[ h_{t+1} - h_t \big|\mathcal{F}_t \right] &\leq 2 \gamma (r_t - r^*)^T \nabla \Pi^F(r_t)
    + 2\gamma\bar{b}\lVert\nabla^2\Pi^F\rVert_\infty\lVert K-\indicator{\lR^+}\rVert_1
    + \gamma_t^2 \big(1 + \bar{b} \big(1+\|\nabla \Pi^F \|_\infty\big) \|k\|_\infty\big).
    \end{aligned}
\end{equation}
Denoting $C(\gamma, k) = 2\gamma\bar{b}\lVert\nabla^2\Pi^F\rVert_\infty\lVert K-\indicator{\lR^+}\rVert_1
    + \gamma^2 \big(1 + \bar{b} \big(1+\|\nabla \Pi^F \|_\infty\big) \|k\|_\infty\big)$ and $\bar{h}_t = \mathbb{E}(h_t)$, and then taking the expectation in Eq.~\ref{eq:proof.bach.1}, one obtains:
\begin{equation}
    \label{eq:proof.bach.2}
    \bar{h}_{t+1} \leq (1 - 2 \gamma_t \tilde{\mu}) \bar{h}_t + C(\gamma, k).
\end{equation}
Further, Eq.~\ref{eq:proof.bach.2} is exactly the same as Eq.~25 in \cite{bach2011non} with different definitions for the constants, and the rest of the proof is identical. As a result, one has:
\begin{align*}
    \bar{h}_t \leq \big(\bar{h}_0 + C(\gamma, k) (t-1)\big) \exp\Big(-\frac{\tilde{\mu}\gamma}{2} t \Big) + \frac{2 C(\gamma, k) }{\tilde{\mu}}.
\end{align*}
\end{proof}

\trackingregret*

\begin{proof}
Denoting $\{[s_i, t_i]\}_{i=1}^\tau$ the intervals on which the distribution is constant, we have
\begin{align*}
    R(T)
    & = \lE\left(\sum_{t=1}^T \Pi^{F_t}(r_t^*) - \Pi^{F_t}(r_t)\right) \\
    & \leq \lE\left(\sum_{t=1}^T \frac{\lVert\nabla^2\Pi^{F_{t}}\rVert_\infty}{2}\lVert r_t^* - r_t\rVert_2^2\right)\\
    & \leq \sum_{i=1}^\tau\frac{\lVert\nabla^2\Pi^{F_{s_i}}\rVert_\infty}{2}\sum_{t=s_i}^{t_i}\lE\left(\lVert r_{s_i}^* - r_t\rVert_2^2\right)\\
    & \leq \sum_{i=1}^\tau\frac{\lVert\nabla^2\Pi^{F_{s_i}}\rVert_\infty}{2}\sum_{t=1}^{t_i -s_i +1}\lE\left(\lVert r_{s_i}^* - r_{t+s_i-1}\rVert_2^2\right)
\end{align*}
Applying Thm.~\ref{thm:tracking} on $\lE\left(\lVert r_{s_i}^* - r_{t+s_i-1}\rVert_2^2\right)$ and denoting $\bar{C}(\gamma, k) = \max_{i\in[\tau]}2\gamma\bar{b}\lVert\nabla^2\Pi^{F_{s_i}}\rVert_\infty\lVert K-\indicator{\lR^+}\rVert_1
    + \gamma^2 \big(1 + \bar{b} \big(1+\|\nabla \Pi^{F_{s_i}} \|_\infty\big) \|k\|_\infty\big)$ we obtain
\begin{align*}
    R(T)
    & \leq \sum_{i=1}^\tau\frac{\lVert\nabla^2\Pi^{F_{s_i}}\rVert_\infty}{2}\sum_{t=1}^{t_i -s_i +1} \left(\big(\bar{b}^2 + \bar{C}(\gamma, k) (t-1)\big) e^{-\frac{\mu c\gamma}{2} t} + \frac{2 \bar{C}(\gamma, k) }{\mu c}\right) \\
    & \leq  
    \frac{1}{2}\max_{i\in[\tau]}\lVert\nabla^2\Pi^{F_{s_i}}\rVert_\infty
    \left(\frac{2 \bar{C}(\gamma, k) }{\mu c} T 
    +
    \sum_{i=1}^\tau\sum_{t=1}^{t_i -s_i +1} \bar{b}^2 e^{-\frac{\mu c\gamma}{2} t} + \bar{C}(\gamma, k) (t-1)e^{-\frac{\mu c\gamma}{2} t}
    \right)\\
    & \leq
    \frac{1}{2}\max_{i\in[\tau]}\lVert\nabla^2\Pi^{F_{s_i}}\rVert_\infty
    \left(\frac{2 \bar{C}(\gamma, k) }{\mu c} T 
    +
    \sum_{i=1}^\tau\sum_{t=0}^{t_i -s_i} \bar{b}^2 e^{-\frac{\mu c\gamma}{2} (t+1)} + \bar{C}(\gamma, k) te^{-\frac{\mu c\gamma}{2} (t+1)} 
    \right)\\
    & \leq 
    \frac{1}{2}\max_{i\in[\tau]}\lVert\nabla^2\Pi^{F_{s_i}}\rVert_\infty
    \left(
    \frac{2 \bar{C}(\gamma, k) }{\mu c} T 
    +
    \sum_{i=1}^\tau\left(\frac{\bar{b}^2 e^{-\frac{\mu c\gamma}{2}}}{1-e^{-\frac{\mu c\gamma}{2}}} + \sum_{t=0}^{t_i -s_i} \bar{C}(\gamma, k) te^{-\frac{\mu c\gamma}{2} (t+1)}\right) 
    \right)\\
    & \leq 
    \frac{1}{2}\max_{i\in[\tau]}\lVert\nabla^2\Pi^{F_{s_i}}\rVert_\infty
    \left(
    \frac{2 \bar{C}(\gamma, k) }{\mu c} T 
    +
    \left(\frac{\bar{b}^2}{e^{\frac{\mu c\gamma}{2}}-1}\tau 
    + \bar{C}(\gamma, k)e^{-\mu c\gamma}\sum_{i=1}^\tau \frac{1 - e^{-\frac{\mu c \gamma}{2}(t_i-s_i)}\left(1 + (t_i-s_i)\left(1 - e^{-\frac{\mu c \gamma}{2}}\right)\right)}{\left(1 - e^{-\frac{\mu c \gamma}{2}}\right)^2}\right) 
    \right)\\
    & \leq 
    \frac{1}{2}\max_{i\in[\tau]}\lVert\nabla^2\Pi^{F_{s_i}}\rVert_\infty
    \left(\frac{2 \bar{C}(\gamma, k) }{\mu c} T 
    +
    \left(\frac{\bar{b}^2}{e^{\frac{\mu c\gamma}{2}}-1}\tau 
    + \frac{\bar{C}(\gamma, k)e^{-\mu c\gamma}}{\left(1 - e^{-\frac{\mu c \gamma}{2}}\right)^2}\tau\right) \right)\\
    & \leq 
    \frac{1}{2}\max_{i\in[\tau]}\lVert\nabla^2\Pi^{F_{s_i}}\rVert_\infty
    \left(
    \frac{2 \bar{C}(\gamma, k) }{\mu c} T 
    +
    \left(\frac{\bar{b}^2}{e^{\frac{\mu c\gamma}{2}}-1} 
    + \frac{\bar{C}(\gamma, k)}{\left(e^{\frac{\mu c \gamma}{2}}-1\right)^2}\right) \tau
    \right)
\end{align*}

Getting $R(T) = \cO(\sqrt{T})$ when $T$ is known in advance just amounts to plugging $\gamma = \frac{1}{\sqrt{T}}$, $\lVert K - \indicator{\lR^+}\rVert_1 \propto \frac{1}{\sqrt{T}}$ and $\|k\|_\infty \propto \sqrt{T}$ in the last equation.
\end{proof}

\end{document}